\documentclass[twoside]{article}

\usepackage[accepted]{aistats2026}

\usepackage[round]{natbib}

\usepackage[utf8]{inputenc} %
\usepackage[T1]{fontenc}    %
\usepackage{hyperref}       %

\usepackage{url}            %
\usepackage{booktabs}       %
\usepackage{amsfonts}       %
\usepackage{nicefrac}       %
\usepackage{microtype}      %
\usepackage{xcolor}         %

\usepackage{tikz}
\usetikzlibrary{automata, positioning}

\usepackage{graphicx} %
\usepackage[utf8]{inputenc} %
\usepackage[T1]{fontenc}    %
\usepackage{hyperref}       %
\usepackage{url}            %
\usepackage{booktabs}       %
\usepackage{amsfonts}       %
\usepackage{nicefrac}       %
\usepackage{microtype}      %
\usepackage{xcolor}         %

\usepackage[utf8]{inputenc} %
\usepackage[T1]{fontenc} 
\usepackage{hyperref} %
\usepackage{amsmath}%
\usepackage{amsthm}
\usepackage{amssymb}
\usepackage{paralist}
\usepackage{enumitem}
\usepackage{algorithm}
\usepackage{algorithmicx}
\usepackage{algpseudocode}
\usepackage{xcolor}
\usepackage{mathtools}
\usepackage{dsfont}
\providecommand{\mathbbm}[1]{\mathds{#1}}
\usepackage{subcaption}
\usepackage{yfonts}

\usepackage{tikz}
\usetikzlibrary{automata, positioning, arrows}

\DeclareMathOperator*{\argmax}{arg\,max}
\DeclareMathOperator*{\argmin}{arg\,min}

\newcommand{\bmu}{\boldsymbol{\mu}}

\newcommand{\bo}{\boldsymbol{\omega}}

\newcommand{\brho}{\boldsymbol{\rho}}

\newcommand{\snc}{\sigma_{\textnormal{NC}}}
\newcommand{\uno}{u_{\textnormal{NO}}}

\newcommand{\test}{\textnormal{Test}}

\newcommand{\calA}{\mathcal{A}}
\newcommand{\calS}{\mathcal{S}}

\newcommand{\Ptest}{\mathcal{P}_{\test}}

\newcommand{\proj}{\textnormal{\bf proj}}

\newcommand{\RR}{\mathbb{R}}
\newcommand{\PP}{\mathbb{P}}
\newcommand{\EE}{\mathbb{E}}

\newcommand{\NN}{\mathbb{N}}
\newcommand{\alt}{\textnormal{Alt}}
\newcommand{\ans}{\textnormal{Ans}}

\newcommand{\kl}{\textnormal{KL}}
\newcommand{\cl}{\textnormal{cl}}
\newcommand{\setmap}{\rightrightarrows}

\newcommand{\skl}{\textnormal{kl}}

\makeatletter
\newtheorem*{rep@theorem}{\rep@title}
\newcommand{\newreptheorem}[2]{%
	\newenvironment{rep#1}[1]{%
		\def\rep@title{#2 \ref{##1}}%
		\begin{rep@theorem}}%
		{\end{rep@theorem}}}
\makeatother

\newtheorem{theorem}{Theorem}
\newtheorem{definition}{Definition}
\newtheorem{lemma}{Lemma}
\newtheorem{remark}{Remark} 
\newtheorem{assumption}{Assumption} 
\newtheorem{proposition}{Proposition}
\newtheorem{corollary}{Corollary}

\begin{document}

\runningauthor{Kaito Ariu\textsuperscript{*}, Po-An Wang\textsuperscript{*}, Alexandre Proutiere, Kenshi Abe}

\twocolumn[

\aistatstitle{Policy Testing in Markov Decision Processes}

\aistatsauthor{Kaito Ariu$^1$\textsuperscript{*} \And   Po-An Wang$^2$\textsuperscript{*}  \And  Alexandre Proutiere$^3$ \And Kenshi Abe$^1$}

\aistatsaddress{$^1$CyberAgent \And $^2$National Tsing Hua University \And $^3$KTH, Digital Futures} 

\vspace{-0.3em}
{\footnotesize\centering \textsuperscript{*}Alphabetical order.\par}
\vspace{0.3em}
]

\begin{abstract}
We study the policy testing problem in discounted Markov decision processes (MDPs) in the fixed-confidence setting under a generative model with static sampling. The goal is to decide whether the value of a given policy exceeds a specified threshold while minimizing the number of samples. We first derive an instance-dependent lower bound that any reasonable algorithm must satisfy, characterized as the solution to an optimization problem with non-convex constraints. Guided by this formulation, we propose a new algorithm. While this design paradigm is common in pure exploration problems such as best-arm identification, the non-convex constraints that arise in MDPs introduce substantial difficulties. To address them, we reformulate the lower-bound problem by swapping the roles of the objective and the constraints, yielding an alternative problem with a non-convex objective but convex constraints. This reformulation admits an interpretation as a policy optimization task in a newly constructed {\it reversed MDP}. We further show that the global KL constraint can be decomposed exactly into a family of product-box subproblems, which are solved by projected policy gradient and combined through an outer budget search. Beyond policy testing, our reformulation and reversed MDP view suggest extensions to other pure exploration tasks in MDPs, including policy evaluation and best policy identification.

\end{abstract}
\section{INTRODUCTION}
Reinforcement learning (RL) commonly models the interaction between a learning agent and its environment as a Markov decision process (MDP) \citep{puterman1994markov}, due to its flexibility and wide applicability. Fundamental problems in RL, such as policy evaluation and best policy identification, have received significant attention, and the performance of learning algorithms on these pure exploration tasks is typically measured by their sample complexity. Ideally, we aim to design algorithms with instance-specific optimal sample complexity. This ensures that the algorithm adapts to the specific problem instance at hand, rather than to a worst-case scenario, and accurately reflects its true difficulty. In the context of multi-armed bandits (MABs), which can be interpreted as stateless RL, the design of algorithms for the best arm identification task is relatively well understood, and several instance-optimal algorithms exist \citep{garivier2016optimal, degenne2019non, wang2021fast}.

However, extending such guarantees to RL settings governed by MDPs is highly non-trivial. The primary challenge is that, unlike in bandits, the set of parameterizations that make two MDP instances hard to distinguish—the so-called \emph{confusing parameters}—forms a non-convex set (\S \ref{sec:lower}). As a result, the optimization problems that characterize instance-specific complexity in RL, also referred to as the lower bound problems, are inherently non-convex and computationally intractable in general. Common workarounds rely on convex relaxations \citep{al2021adaptive}, which compromise statistical optimality.

In this paper, we address this challenge for the \emph{policy testing} problem under discounted, tabular MDPs with access to a generative model and a static sampling allocation: given a confidence level $\delta$, the agent must decide whether the value of a given policy exceeds a specified threshold with probability at least $1-\delta$. Our main contribution is a reformulation that turns the generally non-convex lower bound problem for policy testing into a tractable form without sacrificing statistical optimality. Specifically, we show that by swapping the roles of the objective and the constraints, the problem can be recast as policy optimization in a newly constructed \emph{reversed MDP} (\S \ref{sec:sol to no}), yielding convex constraints and a non-convex objective.
We then derive an exact decomposition of the global KL-feasible set into a family of product-box subproblems, each of which is solved by projected policy gradient, while an outer search optimizes the local KL-budget allocation (Theorem~\ref{thm:proj_pg_rate}).
Combining these pieces, we propose the PTST algorithm (\S \ref{sec:testing alg}) and prove that it attains asymptotically instance-optimal sample complexity (Theorem~\ref{thm:sample}). To the best of our knowledge, this is the first computationally tractable algorithm to achieve instance optimality for pure exploration in MDPs. Beyond policy testing, our reformulation and the reversed MDP perspective suggest extensions to other pure exploration problems in MDPs, including policy evaluation and best policy identification (see Appendix~\ref{app:extension} for a more detailed discussion).

Beyond these theoretical contributions, the problem is practically motivated. Policy testing asks whether a proposed policy meets a target (or improves on a deployed baseline) with high confidence while using as few samples as possible. For example, consider a recommendation system that leverages user history (states), where the algorithm driving recommendations is a policy. When a deployed policy has been used for a long period, its value—i.e., the utility achieved by the system—is typically well understood. Suppose we now design a new policy. Determining whether the value of this new policy exceeds that of the previously deployed one is crucial for maximizing overall benefit. A closely related scenario arises in healthcare, where a policy corresponds to a treatment strategy. More generally, policy testing can be viewed as a stateful analogue of thresholding bandits \citep{locatelli2016optimal}.

\paragraph{Contributions.}
Our contributions are as follows:
(i) We derive an instance-specific lower bound on sample complexity (Theorem~\ref{thm:stat_lower}), revealing non-convex constraints.
(ii) We reformulate the lower-bound optimization as an equivalent reversed MDP, derive an exact decomposition into product-box-constrained subproblems, and solve them by combining projected policy gradient with an outer budget search (Theorem~\ref{thm:proj_pg_rate}).
(iii) Building on this, we develop PTST and prove that it is asymptotically instance-optimal in sample complexity (Theorem~\ref{thm:sample}).
(iv) Empirically, across the evaluated instances, PTST requires fewer samples at a fixed confidence level than the adapted best policy identification baseline.

\section{RELATED WORK}

Pure exploration in MABs has been studied extensively. In particular, significant attention has been devoted to best-arm identification in both the fixed-confidence and fixed-budget settings (see, e.g., \cite{audibert2010best, gabillon2012best, soare2014best}). In the fixed-confidence setting, instance-specific lower bounds on sample complexity have been derived. These bounds, in turn, have enabled the design of asymptotically optimal algorithms \citep{garivier2016optimal, degenne2019non, jedra2020optimal, wang2021fast}. This line of analysis is feasible thanks to the relative simplicity of the optimization problems that underlie these lower bounds.

Similar challenges have been studied in the context of MDPs. Several works focus on characterizing the minimax sample complexity for best policy identification, typically under the generative model assumption \citep{azar2013,agarwal20b,li2022}. Other efforts aim at instance optimality either in offline settings \citep{khamaru2021, wang2024optimal} or under adaptive sampling regimes \citep{zanette2019almost, al2021adaptive, al2021navigating, tirinzoni2022near, kitamura2023regularization, taupin2023best, russo2024multi}. However, none of these approaches achieves true instance-specific optimality. Even in the relatively simple case of tabular episodic MDPs, current results attain only near-optimal sample complexity \citep{tirinzoni2022near, al2023towards, narang2024sample}. The core barrier to achieving instance-specific optimality in MDPs lies in the inherent complexity of the optimization problem that defines the sample-complexity lower bound. In this paper, we provide the first approach to fully deal with this complexity in the context of the policy testing task. A more detailed discussion of related work can be found in Appendix~\ref{app:further_relatedwork}.

\section{PRELIMINARIES}
\subsection{Markov Decision Processes}\label{subsec:MDP}
We consider a discounted Markov decision process (MDP) ${\cal M}=({\cal S}, {\cal A}, p, r, \brho, \gamma)$, where ${\cal S}$ and ${\cal A}$ denote the finite state and action spaces, respectively. The (unknown) transition kernel is given by $p \in \mathcal{P} := \Delta({\cal S})^{{\cal S}\times {\cal A}}$, where $\Delta(\mathcal{X})$ denotes the simplex over $\mathcal{X}.$ $p(s'|s,a)$ denotes the probability to move to state $s'$ given the current state $s$ and the selected action $a$. The reward function $r:{\cal S} \times {\cal A} \to \mathbb{R}$ is deterministic, $\brho$ represents the known initial state distribution, and $\gamma \in (0,1)$ is the discount factor.
 We denote the state-action pair at time $t$ by $(s(t), a(t))$. At time $t$, the agent selects action $a(t)$ according to the distribution $\pi(\cdot\ | s(t))$, collects reward $r(s(t), a(t))$, and moves to the next state $s(t+1)$ according to the distribution $p(\cdot\ |\ s(t),\ a(t))$. The value function of a given policy $\pi\in \Pi:=  \Delta({\cal A})^{\cal S}$ is defined by its average long-term discounted reward given any possible starting state $s$: 
 $
 V_{p}^\pi (s):=\EE_{p}^\pi \Big[\sum_{t\ge 0}\gamma^t r(s(t),a(t))\mid s(0)=s\Big],
$ 
 where $\EE^\pi_p$ represents the expectation taken with respect to randomness induced by $\pi$ and $p$. 
 Similarly, for each state-action pair $(s,a)\in\mathcal{S}\times \mathcal{A}$, Q-function is defined as:
 $
 Q_{p}^\pi (s,a):=\EE_{p}^\pi\Big[\sum_{t\ge 0}\gamma^t r(s(t),a(t))\mid s(0)=s, a(0)=a\Big].
 $
 The value of $\pi$ is then defined as:
 $
 V^\pi_{p}(\brho):=\sum_{s\in\mathcal{S}}\rho_sV^\pi_{p}(s).
 $
 We also define the discounted state-visitation distribution: $d_{p,s,a}^\pi(s',a')=(1-\gamma)\EE_{p}^\pi\Big[\sum_{t\ge 0} \gamma^t \mathbbm{1}\{(s(t),a(t))=(s',a')\mid (s(0),a(0))=(s,a)\}\Big]$
 and $d_{p,s}^\pi(s') =  \sum_{a \in \mathcal{A}}\pi(a|s)d_{p,s,a}^\pi(s',a')$. The state-visitation distribution initialized by $\brho$, $d_{p, \brho}^\pi \in \Delta(\calS)$, is defined with components $ d_{p,\brho,s}^\pi = \sum_{s' \in \calS} \rho_{s'} d_{p,s'}^\pi(s)$ for each $s\in\mathcal{S}$. For any $\brho, \bmu \in \Delta(\calS)$, we define $\|\brho/\bmu\|_\infty:= \max_{s \in \calS} \rho_s/\mu_s$, with $0/0=1$ by convention. We have $d_{p,\brho,s}^\pi \ge (1 -\gamma)\rho_s$ and $ \| d_{p,\brho}^\pi/  d_{p,\brho}^{\pi'}\|_\infty \le\| d_{p,\brho}^\pi/  \brho\|_\infty/(1-\gamma)$ for all $\brho \in \Delta(\calS)$, $ \pi, \pi' \in \Pi$.

\subsection{Policy Testing}
We aim to devise an algorithm that determines whether the value $V^\pi_{p}(\brho)$ of a given policy $\pi$ exceeds some given threshold with a minimal number of samples.  Without loss of generality, we can assume that this threshold is 0.\footnote{If the threshold is $R$, we can instead use the shifted reward function $\tilde{r} = r - (1-\gamma)R$, and the new value function is
$
\tilde{V}_{p}^{\pi}(s) = V_{p}^{\pi}(s) - R.
$
Therefore, testing $V_{p}^{\pi}(\brho) > R$ is equivalent to testing $\tilde{V}_{p}^{\pi}(\brho) > 0$.}
We assume that the kernel $p$ should satisfy $V^\pi_{p}(\brho)\neq 0$, i.e., the value is strictly positive or negative.  
Therefore, we write the set of problem instances: $\mathcal{P}_{\test}:=\{q\in\mathcal{P}:V^\pi_{q}(\brho)\neq 0\}$. For each $p\in \Ptest$, the answer $\text{Ans}(p)$ is $+$ if $V_p^\pi(\brho) >0$ and $-$ if $V_p^\pi(\brho) < 0$.%

We assume that the agent has access to a generative model. In each step, the agent selects a state-action pair, from which the transition to the next state is observed. We consider the case where the agent uses a static sampling rule, targeting fixed proportions of state-action draws $\bo\in \Sigma:=\{\bo'\in [0,1]^{|\mathcal{S}|\times |\mathcal{A}|}:\sum_{s,a}\omega_{sa}'=1\}$  ($\omega_{sa}$ denotes the proportion of time state-action pair is sampled). 
Our goal is to design an algorithm that, with a fixed confidence level of $1-\delta$ (where $\delta \in (0,1)$ is a predefined  parameter), determines as quickly as possible whether $V^\pi_{p}(\brho)$ exceeds the given threshold (i.e., whether $V^\pi_{p}(\brho)>0$ or $V^\pi_{p}(\brho)<0$).

\begin{remark}[Generative Model]
Much of the sample-complexity analysis for MDPs has been conducted under the assumption of access to a generative model \citep{azar2013,zanette2019almost,al2021adaptive}. While one can extend the analysis to the forward (online, single-trajectory) interaction model, such extensions typically require additional assumptions. See Appendix~\ref{app:extension} for details.
\end{remark}

In addition to a sampling rule, the algorithm consists of a stopping rule and a decision rule. The stopping rule is defined through a stopping time $\tau$ w.r.t. the natural filtration ${\cal F}=({\cal F}_t)_{t\ge 1}$, where ${\cal F}_t$ denotes the $\sigma$-field generated by all the observations collected up to and including round $t$. In round $\tau$, after stopping, the algorithm returns a ${\cal F}_{\tau}$-measurable decision $\hat{\imath}\in \{+, -\}$, corresponding to the answer which is believed to be correct. The sample complexity of an algorithm is defined as $\mathbb{E}_p[\tau]$ where the expectation is with respect to the sampling process, the observations, and the stopping rule. 
\begin{definition}\label{def:delta-PC}
  An algorithm is  $\delta$-Probably Correct ($\delta$-PC) if for all $p \in \Ptest$, (i) it stops almost surely, $\PP_p[\tau <\infty]=1$ and (ii) $\PP_p[\hat{\imath}\neq \ans(p)]\le \delta$. 
\end{definition}

We aim to design a $\delta$-PC algorithm with minimal sample complexity.

\subsection{Assumptions}
To simplify the notation, we define $r_{\max}:=\max_{s,a}|r(s,a)|$, $r^\pi(s):=\sum_{a\in\mathcal{A}}\pi(a\vert s)r(s,a)$ and $r^\pi (\brho):=\sum_{s\in\mathcal{S}}\rho_sr^\pi(s)$. 

As 
$
V_p^\pi(\brho) = r^\pi (\brho) + \sum_{t=1}^\infty\mathbb{E}_p^\pi[\gamma^t r(s(t), a(t))],$ the transition kernel $p$ maximizing the value maps all state-action pairs to the most rewarding state, $\argmax_s r^\pi(s)$. In contrast, the kernel minimizing the value maps all state-action pairs to the least rewarding state, $\argmin_s r^\pi(s)$. That is,
\begin{align}\label{eq:maxmin V}
&\max_pV_p^\pi(\brho) = r^\pi (\brho) + \frac{\gamma}{1-\gamma}\max_{s}r^\pi(s)\\
\hbox{ and }&\min_pV_p^\pi(\brho) = r^\pi (\brho) + \frac{\gamma}{1-\gamma}\min_{s}r^\pi(s) .  
\end{align}

Throughout the paper we impose the following standing assumption, which ensures that both decision regions
$\{q\in\Ptest:\ans(q)=-\}$ and $\{q\in\Ptest:\ans(q)=+\}$ are nonempty and simplifies the exposition.
\begin{assumption}\label{apt:general}
$\rho_s > 0$ for all $s \in {\cal S}$. 
    $r$ and $\brho$ satisfy: %
    $\frac{-\gamma}{1-\gamma}\min_sr^\pi(s)>r^\pi (\brho)>\frac{-\gamma}{1-\gamma}\max_sr^\pi(s)$. 
\end{assumption}

This assumption also implies that for any transition kernel, the state value function is not constant (it varies across states). This is formalized in the following lemma, proved in Appendix~\ref{subsec:from apt}. 

\begin{lemma}\label{lem:from apt}
Under Assumption~\ref{apt:general}, 
$$
\min_{q\in \mathcal{P}}\max_{s,s'\in\mathcal{S}}V^\pi_q(s)-V^\pi_q(s')>0.
$$
\end{lemma}
Throughout the paper we also adopt the following assumption. We take the target policy $\pi$ to have full support and require our static sampling rule to explore every state–action pair in the support of $\pi$. This is without loss of generality: if $\pi(a\mid s)=0$ for some $(s,a)$, then the transition law $p(\cdot\mid s,a)$ does not affect $V_p^\pi$ and such pairs can be ignored.

\begin{assumption}\label{apt:sampling} 
$\pi(a\mid s)>0$ and $\omega_{sa}>0$ for all $s,a \in \mathcal{S}\times\mathcal{A}$.
\end{assumption}

\section{LOWER BOUND}\label{sec:lower}

We derive sample complexity lower bounds satisfied by any $\delta$-PC algorithm. To this aim, we leverage the classical change-of-measure arguments in MAB \citep{lai1985asymptotically,garivier2016optimal}. To state our lower bound, we need the following notation. For any state-action pair $(s,a)$, $\kl_{sa}(p,q)$ denotes the KL divergence between the distributions $p(\cdot\mid s,a)$ and $q(\cdot\mid s,a)$. Finally, for $t\in \NN,\,(s,a)\in \mathcal{S}\times \mathcal{A}$, $N_{sa}(t)$ denotes the number of times $(s,a)$ is sampled up to $t$.

We introduce the set of {\it alternative or confusing} kernels as $\alt (p):=\{q\in {\cal P}_{\test}:\ans (p)\neq \ans(q)\}$. This set collects all the kernels for which the answer to the test differs from $p$. $\alt (p)$ can also be written as follows:
$
\alt (p)=  \{ q \in \Ptest : V_q^\pi(\brho)V_p^\pi(\brho)<0\}$.
\begin{theorem}\label{thm:stat_lower}
Under Assumption~\ref{apt:general}, let $p\in \Ptest,\,\bo\in \Sigma$, and a $\delta$-PC algorithm with sampling rule satisfying that for any $\varepsilon>0$, there exists $c_\varepsilon>0$ such that $\mathbb{E}[N_{sa}(t)]\le t(\omega_{sa}+\varepsilon)+c_{\varepsilon},\,\forall s\in\mathcal{S},a\in\mathcal{A}$. Then,
\begin{equation}
   \liminf_{\delta\to 0} \frac{\EE_p[\tau]}{\log (1/\delta)}\ge T_{\bo}^\star(p),
\end{equation}
\text{where}
\begin{equation}\label{eq:Tbo}
T_{\bo}^\star(p)^{-1}:=\inf_{q\in \alt (p)} \sum_{s,a}\omega_{sa}\kl_{sa}(p,q).
\end{equation}
\end{theorem}

Theorem~\ref{thm:stat_lower} is proved in Appendix~\ref{app:lower}.  The next result, proved in Appendix~\ref{subsec:finite T}, states that under Assumption~\ref{apt:sampling}, the characteristic time $T_{\bo}^\star(p)$ is finite.

\begin{proposition}\label{prop:finite T}
  If Assumption~\ref{apt:sampling} holds,  $T_{\bo}^\star(p)$ is finite.%
\end{proposition}

Most existing asymptotic optimal algorithms for pure exploration in MAB involve solving the optimization problem (\ref{eq:Tbo}) \citep{garivier2016optimal,degenne2019non,wang2021fast}. Using a certain threshold parameter $\beta(t,\delta)$, determining whether this optimization problem exceeds $\beta(t,\delta)/t$ becomes the key to deriving the optimal stopping rule. 
However, for MDPs, the optimization problem leading to the sample complexity lower bound is non-convex as shown below.

\paragraph{An example where $\alt(p)$ is non-convex.} Let $\mathcal{M}$ be a MDP which consists of three states $s_1,s_2,s_3$, and $\pi$ be a deterministic policy such that $\pi(a\vert s_i)=1$ for all $i=1,2,3$. The initial distribution, discount factor, and reward function are set as: $\brho=(1/3,1/3,1/3)$, $\gamma=0.9,\,r(a\vert s_1)=-0.88,\,r(a\vert s_2)=r(a\vert s_3)=0.12$. We define the transition kernels $p,q^{(1)},q^{(2)}$ as
\begin{align*}
[q^{(1)}_{ij}]=\left(\begin{array}{ccc}
    0&1&0\\
    0&1&0\\
    0&0&1
\end{array}\right),\,
[q^{(2)}_{ij}]=&\left(\begin{array}{ccc}
    0&0.5&0.5\\
    1&0&0\\
    0&0&1
\end{array}\right),
\end{align*}
and $p=(q^{(1)}+q^{(2)})/2$, 
where $q^{(1)}_{i,j}$ is the abbreviation for $ (q^{(i)}(s_j\vert s_i,a))$, and likewise for $q^{(2)}_{ij}$. One can see that %
$V^{\pi}_p(\brho)\approx -0.15<0,V^{\pi}_{q^{(1)}}(\brho)\approx 0.87>0,V^{\pi}_{q^{(2)}}(\brho)\approx 0.13>0$. Hence $q^{(1)}, q^{(2)}\in \alt (p)$ but $(q^{(1)}+q^{(2)})/2=p\notin \alt (p)$.

\section{REVERSED MDP}\label{sec:sol to no}

In this section, we present novel ideas for constructing an optimal stopping rule for policy testing. We show that the problem can be viewed as policy optimization in a reversed MDP, where the roles of the transition kernel and the policy are interchanged.

\subsection{Non-Convex Constraint in Stopping Rule}
We begin with a standard fixed-confidence stopping rule and show that it leads to an optimization problem with non-convex constraints. 
Let $N_{sa}(t)$ denote the number of times the state–action pair $(s,a)$ has been sampled up to round $t$. 
Let the threshold $\beta(t,\delta)$ be defined by:
\begin{equation}
\beta(t,\delta):=\log\left(\frac{1}{\delta}\right)+(\left|\mathcal{S}\right|-1)\sum_{s,a}\log \left(e\left[1+\frac{N_{sa}(t)}{\left|\mathcal{S}\right|-1}\right]\right).
\end{equation}
With the convention that $N_{sa}(t)\kl_{sa}(\hat{p}_t,p)=0$ whenever $N_{sa}(t)=0$, we claim that the stopping rule:
\begin{equation}\label{eq:ideal stopping}
   \inf_{q\in \alt (\hat{p}_t)} \sum_{s,a}N_{sa}(t)\kl_{sa}(\hat{p}_t,q)\ge \beta(t,\delta),
\end{equation}
yields a $\delta$-PC algorithm. 

First, according to Proposition 1 in \cite{jonsson2020planning} and Lemma 15 in \cite{al2021adaptive}, we have: for each $p \in\Ptest$,
\begin{align}\label{eq:beta guarantee}
\PP_{p}\left[\exists t\ge 1,\, \sum_{s,a}N_{sa}(t)\kl_{sa}(\hat{p}_t,p)\ge \beta(t,\delta)\right]\le \delta,
\end{align}
If the answer induced by the current empirical model $\hat p_t$ is incorrect, i.e., $\ans(\hat p_t)\neq \ans(p)$ (equivalently, $p\in \alt(\hat p_t)$), the algorithm should continue sampling. Moreover, if $p\in \alt(\hat p_t)$ and \eqref{eq:ideal stopping} holds, then $p$ is feasible for the infimum in \eqref{eq:ideal stopping}, and therefore
\[
\sum_{s,a} N_{sa}(t)\,\kl_{sa}(\hat p_t,p)\ \ge\ \beta(t,\delta).
\]
By \eqref{eq:beta guarantee}, this event occurs with probability at most $\delta$. Hence, stopping at the first time $t$ that satisfies \eqref{eq:ideal stopping} yields a $\delta$-PC algorithm.
Unfortunately, evaluating \eqref{eq:ideal stopping} is computationally difficult because $\alt(\hat{p}_t)$ is generally non-convex.

\subsection{From Non-Convex Constraint to Non-Convex Objective}\label{subsec:nc and no}

In what follows, we transform the optimization problem into an equivalent one with a non-convex objective and convex constraints.

We first introduce a parameterized extension of the optimization problem defining our sample complexity lower bound (Theorem~\ref{thm:stat_lower}):
\begin{align}\label{opt:nc up}
 \tag{NC-$u,\bo,p$}   \inf_{q}  \sum_{s,a} \omega_{sa}\kl_{sa}(p,q) & \quad \text{s.t. }
   V^\pi_q(\brho) V^\pi_p(\brho)<u,\nonumber
\end{align}
where $u\in\mathbb{R}$. The problem \eqref{opt:nc up} has a non-convex constraint, and we denote its value by $\snc(u,\bo,p)$. With this notation, the stopping rule \eqref{eq:ideal stopping} is 
$$
\snc\!\left(0,\hat{\bo}(t),\hat{p}_t\right) \ge \frac{\beta(t, \delta)}{t}.
$$

Next, we define $\uno(\sigma,\bo,p)$ as the value of the following problem:
\begin{align}\label{opt:no sp}
\tag{NO-$\sigma,\bo,p$}  \min_{q\in {\cal P}}V^\pi_q(\brho) V^\pi_p(\brho) & \quad \text{s.t. } \sum_{s,a} \omega_{sa}\kl_{sa}(p,q)\le \sigma.
\end{align}
In \eqref{opt:no sp}, the objective is non-convex, while the constraint set is convex.

The next proposition, proved in Appendix~\ref{app:dual},  formalizes the bijective relationship between the values $\uno(\sigma,\bo,p)$ and $\snc(u,\bo,p)$ associated with the problems \eqref{opt:no sp} and \eqref{opt:nc up}, respectively.

\begin{proposition}\label{prop:bijection in main}
Suppose that Assumption~\ref{apt:general} holds and that $p\in\Ptest$. Then: for all $\sigma\ge 0$ such that $\uno(\sigma,\bo,p)> \min_{q\in\mathcal{P}}V^\pi_p(\brho)V^\pi_q(\brho)$
\begin{align*}
 \snc(\uno(\sigma,\bo,p),\bo,p)=\sigma,
\end{align*}
for all $u\in (\min_{q\in\mathcal{P}}V^\pi_p(\brho)V^\pi_q(\brho),\uno(0,\bo, p)]$, 
\begin{align*}
\uno(\snc(u,\bo,p),\bo,p)=u.
\end{align*}
\end{proposition}
\begin{proof}[Proof Sketch of Proposition~\ref{prop:bijection in main}]
We can replace the infimum and strict inequality in (\ref{opt:nc up}) with a minimum and a non-strict inequality, respectively. From this, the result follows: (i) If $\snc(u,\bo,p)\le \sigma$, then there exists $q$ such that $V^\pi_p(\brho)V^\pi_q(\brho)\le u$ and $\sum_{sa}\omega_{sa}\kl_{sa}(p,q)\le \sigma$, implying $\uno(\sigma,\bo,p)\le u$. Conversely, if $\uno(\sigma,\bo,p)\le u$, then $\snc(u,\bo,p)\le \sigma$, which is equivalent to stating that if $\snc(u,\bo,p) > \sigma$, then $\uno(\sigma,\bo,p) > u$. (ii) We further show that if $\snc(u,\bo,p)\ge \sigma$, then $\uno(\sigma,\bo,p)\ge u$. Combining (i) and (ii) directly implies that if $\snc(u,\bo,p) = \sigma$, then $\uno(\sigma,\bo,p) = u$.
\end{proof}

This proposition shows that the mappings $\uno(\cdot,\bo,p)$ and $\snc(\cdot,\bo,p)$ are inverses of each other. While this may appear intuitive, it does not generally hold in non-convex settings. We state general conditions in Assumption~\ref{ass:hg_new} (in Appendix~\ref{app:dual}) on the objective and the constraint set that ensure this inverse relationship, and we verify that these conditions are met in our setting.

Using Proposition~\ref{prop:bijection in main}, the stopping rule \eqref{eq:ideal stopping}—namely,
$\snc\!\left(0,\hat{\bo}(t),\hat{p}_t\right)\ge \beta(t,\delta)/t$—is equivalent to
\begin{align*}
\uno(\beta(t,\delta)/t, \hat{\bo}(t), \hat{p}_t)\ \ge\ 0,
\end{align*}
since $\uno(\cdot,\bo,p)$ is nonincreasing in its first argument for any fixed $\bo$ and $p$.

\begin{remark}
Proposition~\ref{prop:bijection in main} is a central component of our approach to solving the optimization problem using the reversed MDP formulation, and thereby to developing an instance-optimal algorithm. While we establish this result for policy testing, we also show that it holds for policy evaluation. Extending this result to other pure exploration tasks, such as best policy identification, remains an interesting direction for future work. See Appendix~\ref{app:extension} for details.
\end{remark}

\subsection{The Reversed MDP}\label{subsec:reversed MDP}

We can interpret the dual optimization problem \eqref{opt:no sp} as a policy optimization problem in a new MDP. This MDP $\bar{\cal M}=(\bar{\mathcal{S}}, \bar{\cal A}, \bar{p}, \bar{r}, \bar{\brho}, \gamma)$ is referred to as reversed MDP, since the roles of policy $\pi$ and transition kernel $p$ are swapped. $\bar{\cal M}$ is constructed as follows. The state and action spaces are $\bar{\mathcal{S}}:=\mathcal{S}\times \mathcal{A}$ and $\bar{\mathcal{A}}=\mathcal{S}$. The initial state distribution $\bar{\brho}$ is such that for all $(s,a)$, $\bar{\brho}(s,a)=\rho_s\pi(a\mid s)$. In state $\bar{s}=(s,a)\in \bar{\mathcal{S}}$, a policy $\bar{\pi}$ takes an action $\bar{a}=s'\in \bar{\mathcal{A}}$ with probability $p(s'| s,a)$. Given an action $\bar{a}=s'$ selected in $\bar{s}$, the system moves to state $\bar{s}'=(s',a')$ with probability $\pi(a'| s')$ (all other transitions occur with probability 0), so that $ \bar{p}(\bar{s}'=(s'',a') | \bar{s}, \bar{a}=s') = \pi(a'| s')1_{\{s''=s' \}}$. The reward function, $\bar{r}:\bar{\mathcal{S}}\times \bar{\mathcal{A}}\to \RR$ is defined as $\bar{r}(\bar{s},\bar{a})=r(s,a)$ if $\bar{s}=(s,a)$. The reversed MDP $\bar{\cal M}$ is illustrated in Figure \ref{fig:reversed}.

\begin{figure}
    \centering
    \includegraphics[width=0.6\linewidth]{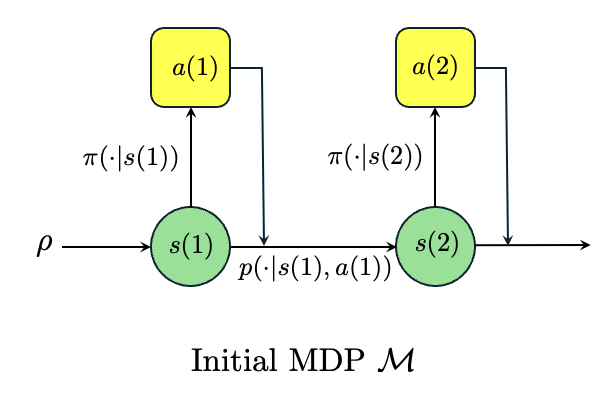}
    \includegraphics[width=0.6\linewidth]{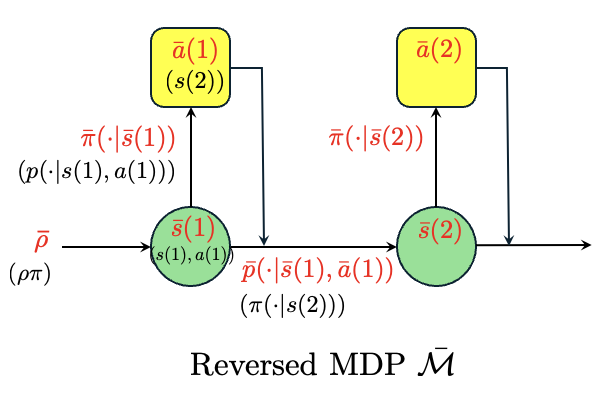}
    \caption{From the original MDP (top) to the reversed MDP (bottom). In the reversed MDP, reversed variables are shown in red and the original counterparts in black.}
    \label{fig:reversed}
    \vspace{-0.5cm}
\end{figure}

For the reversed MDP, the discounted state-visitation distribution starting at $\bar{s}\in\bar{\mathcal{S}}$ is defined as: $\forall \bar{s}'\in\bar{\mathcal{S}},$
\begin{align*}
d^{\bar{\pi}}_{\bar{p},\bar{s}}(\bar{s}'):= (1-{\gamma})\EE^{\pi}_{p}[\sum_{t=0}^\infty {\gamma}^t \mathbbm{1}\{\bar{s}(t)=\bar{s}'\}\mid \bar{s}(0)=\bar{s}],
\end{align*}
which is equal to $d^\pi_{p,s,a}(s',a')$ when $\bar{s}=(s,a)$ and $\bar{s}'=(s',a')$. For any $\bmu\in \Delta(\bar{{\cal S}})$, we define $\bar{d}^\pi_{p,\bmu}(\bar{s}'):=\sum_{\bar{s}\in \bar{\mathcal{S}}}\mu_{\bar{s}}\bar{d}^\pi_{p,\bar{s}}(\bar{s}')$.
The state and state-action value functions of $\bar{\cal M}$ are defined as: for all $(\bar{s},\bar{a})$:
\begin{align}
  \label{eq:V bar} & \bar{V}^{\bar{\pi}}_{\bar{p}}(\bar{s}):=\EE^{\pi}_{p}[\sum_{t=0}^\infty \gamma^t \bar{r}(\bar{s}(t),\bar{a}(t))| \bar{s}(0)=\bar{s}],\\
    &\bar{Q}^{\bar{\pi}}_{\bar{p}}(\bar{s},\bar{a}) :=\EE^{\pi}_{p}[\sum_{t=0}^\infty \gamma^t \bar{r}(\bar{s}(t),\bar{a}(t))| (\bar{s}(0),\bar{a}(0))=(\bar{s},\bar{a})]. \label{eq:Q bar} 
\end{align}
For any $\bmu\in \Delta(\bar{{\cal S}})$, we define $\bar{V}^\pi_{p}(\bmu):=\sum_{\bar{s}\in \bar{\mathcal{S}}}\mu_{\bar{s}}\bar{V}^\pi_{p}(\bar{s})$.
Observe that $\bar{V}^{\bar{\pi}}_{\bar{p}}(\bar{s})=Q^\pi_p(s,a)$ if $\bar{s}=(s,a)$, and $\bar{Q}^{\bar{\pi}}_{\bar{p}}(\bar{s},\bar{a})=r(s,a)+\gamma V^\pi_{p}(s')$ if $(\bar{s},\bar{a})=(s,a,s')$. We simply deduce that for each $s\in \mathcal{S}$,
    $V^\pi_p(s)=\sum_{a\in \mathcal{A}}\pi(a| s)\bar{V}^{\bar{\pi}}_{\bar{p}}(s,a)$,
and $V^\pi_{p}(\brho)=\sum_{s\in \mathcal{S}}\sum_{a\in \mathcal{A}}\rho_s\pi(a| s)\bar{V}^{\bar{\pi}}_{\bar{p}}(s,a)=\bar{V}^{\bar{\pi}}_{\bar{p}}(\bar{\rho})$.
Thus, optimizing transition kernel $q$ in \eqref{opt:no sp} is equivalent to optimizing the policy against in the reversed MDP. More precisely, \eqref{opt:no sp} is equivalent to:
\begin{align}
\min_{\bar{\pi}\in {\cal P}}\bar{V}^{\bar{\pi}}_{\bar{p}}(\bar{\brho})  V^\pi_p(\brho)
      & \quad \text{s.t. } \sum_{s,a} \omega_{sa}\kl_{sa}(p,\bar{\pi})\le \sigma.\label{eq:eqPG}
\end{align}
Reformulating \eqref{opt:no sp} as a policy optimization problem in the reversed MDP offers a key advantage: it allows us to leverage recent advances in the convergence analysis of policy gradient methods. In the next subsection, we present several results for the reversed MDP that support the analysis of constrained policy gradient methods.

\subsection{Preliminary Results for Reversed MDP}
We provide preliminary results that aid our analysis and may be of independent interest. These results are the counterparts, for the reversed MDP, of the performance difference lemma \citep{kakade2002approximately}, the policy gradient theorem \citep{sutton1999policy}, and the smoothness lemma \citep{agarwal2021theory} for standard MDPs. 

We begin with the performance difference lemma, which for the reversed MDP coincides with the celebrated simulation lemma \citep{kearns2002near} (see also Lemma A.1 in \cite{vemula2023virtues}).

\begin{lemma}[Simulation/performance difference lemma]\label{lem:simulation}
    For any $p,\tilde{p}\in \mathcal{P}$ and $(s,a)\in\mathcal{S}\times\mathcal{A}$,  
 \begin{align*}
 &     Q^{\pi}_p(s,a)-Q^{\pi}_{\tilde{p}}(s,a)=\frac{\gamma}{1-\gamma}\sum_{s'\in\mathcal{S},a'\in\mathcal{A}}d^\pi_{p,s,a}(s',a') D,\\
&\text{where } D= \sum_{s''\in\mathcal{S}}V^\pi_{\tilde{p}}(s'')\left(p(s''\mid s',a')-\tilde{p}(s''\mid s',a')\right).
 \end{align*}
\end{lemma} 
Lemma~\ref{lem:simulation} directly implies that $Q^\pi_p(s,a)$ is continuous in $p$, a property that is used in the proof of Proposition \ref{prop:bijection in main}.

The next result provides an explicit expression of the gradient $\nabla_{\bar{\pi}} \bar{V}^{\bar{\pi}}_{\bar{p}}(\bar{\brho})$ used in our policy gradient algorithm. 
\begin{lemma}[Policy gradient]\label{lem:reversed_PG}
    For each $s,s',s''\in\mathcal{S},\, a,a'\in \mathcal{A}$, we have 
    \begin{align}
    \frac{\partial Q^\pi_p(s,a)}{\partial p(s''\mid s',a')}
   & =\frac{d^\pi_{p,s,a} (s',a')}{(1-\gamma)}\left(r(s',a')+\gamma V^\pi_{p}(s'') \right), \label{eq:PG1}\\
   \frac{\partial V^\pi_p(\brho)}{\partial p(s'\mid s,a)}
   & =\frac{d_{p,\brho}^\pi(s,a)}{(1-\gamma)}\left(r(s,a)+\gamma V^\pi_p(s')\right).\label{eq:PG2}
    \end{align}
\end{lemma}

The final result concerns the smoothness of the gradient, and it can be established using tools from the theory of the policy gradient \citep{agarwal2021theory}. It will be useful when assessing the convergence rate of our algorithm.

\begin{lemma}[Smoothness]\label{lem:smoothness_parameter}
    For any $p,\tilde{p}\in\mathcal{P}$ and $(s,a)\in\mathcal{S}\times\mathcal{A}$, 
    $$
\left\|\nabla_{p}Q^\pi_p(s,a)-\nabla_{p}Q^\pi_{\tilde{p}}(s,a)\right\|_{2}\le \frac{2\gamma\left|{\cal S}\right|r_{\max}}{(1-\gamma)^3}\left\|p-\tilde{p}\right\|_{2}.
    $$
\end{lemma}

\section{OPTIMAL ALGORITHM}\label{sec:testing alg}
In this section, we present an asymptotically optimal algorithm, Policy Testing with Static Sampling (PTST).
Algorithm~\ref{alg:static_algorithm} provides the pseudocode for PTST.
It has three main components.

(i) Static sampling rule (line~4). It ensures that the algorithm samples state--action pairs according to a predefined allocation vector $\bo\in\Sigma$. At round $t$, the algorithm tracks $\bo$ by sampling the pair that minimizes $\hat{\omega}_{sa}(t)/\omega_{sa}$, where $\hat{\omega}_{sa}(t):=N_{sa}(t)/t$ is the fraction of rounds in which $(s,a)$ has been sampled up to time $t$.

(ii) Stopping rule (lines~3 and~7). This component is inspired by the reversed-MDP formulation (\S\ref{sec:sol to no}).
We use the nested projected-gradient descent method of Algorithm~\ref{alg:constrained_pg} to approximately solve \eqref{opt:no sp} at time $t$, with $p=\hat p_t$, $\bo=\hat{\bo}(t)$, and $\sigma=\beta(t,\delta)/t$.
Specifically, given a tolerance $\zeta_t>0$, Algorithm~\ref{alg:constrained_pg} returns a value $u_{\zeta_t}$ such that
\[
u_{\zeta_t}\ \ge\ \uno\!\left(\tfrac{\beta(t,\delta)}{t},\,\hat{\bo}(t),\,\hat p_t\right)\ \ge\ u_{\zeta_t}-\zeta_t.
\]
The solver searches over local KL-budget vectors and, for each fixed budget vector, runs exact projected gradient descent on the corresponding product feasible set in the reversed MDP.
The algorithm stops when $u_{\zeta_t}\ge \zeta_t$.

(iii) Decision rule (line~10). Upon stopping, PTST outputs $\ans(\hat{p}_\tau)$ as its decision.

\begin{algorithm}
\caption{Policy Testing with Static Sampling (PTST)}\label{alg:static_algorithm}
\begin{algorithmic}[1]
\State {\bf Input:} $\pi\in \Pi,\delta\in (0,1), \bo\in \Sigma$, $\{\zeta_t\}_{t\ge1}$.
\State {\bf Initialization:} Sample each $(s,a)\in \mathcal{S}\times\mathcal{A}$ once if $\omega_{sa}>0$. $t \gets \sum_{s,a} \mathbbm{1}\{\omega_{sa}>0\}$.
\While{$u_{\zeta_t}  - \zeta_t <  0$}
    \State Sample $(s_t,a_t)\gets \argmin_{(s,a): \omega_{sa}>0} N_{sa}(t-1)/\omega_{sa}$ (tie-broken arbitrarily)
    \State $t \gets t+1$
    \State Update $\hat{p}_t$, $N_{sa}(t)$, and $\hat{\bo}(t)$
    \State Run Algorithm~\ref{alg:constrained_pg} with $(\hat{p}_t,\zeta_t,\beta(t,\delta)/t, \hat{\bo}(t))$ as inputs, and let $u_{\zeta_t}$ be its output.
\EndWhile
\State $\tau \gets t$
\State {\bf Output:} $\hat{\imath} \gets \text{Ans}(\hat{p}_{\tau})$
\end{algorithmic}
\end{algorithm}

\subsection{Projected Gradient Descent}
\label{subsec:rates_policy_grad_const_MDP}

In PTST, the value of \eqref{opt:no sp} is approximated by the nested projected-gradient descent method of Algorithm~\ref{alg:constrained_pg}.
For $\sigma>0$, define the weighted budget simplex
\[
\mathcal B_\sigma(\bo):=
\left\{b\in\mathbb{R}_+^{|\mathcal S||\mathcal A|}:\sum_{s,a}\omega_{sa}b_{sa}\le \sigma\right\},
\]
and, for each $b\in\mathcal B_\sigma(\bo)$,
\[
Q(b):=
\left\{q\in\mathcal P:\kl_{sa}(p,q)\le b_{sa},\ \forall (s,a)\in \mathcal S\times\mathcal A\right\}.
\]
The key observation is that the global KL ball is exactly the union of these product boxes.

\begin{proposition}[Exact reduction to product-box budgets]\label{prop:product_box_reduction}
For every $\sigma>0$, $\bo\in\Sigma$, and $p\in\mathcal P$,
\[
\uno(\sigma,\bo,p)
=
\min_{b\in \mathcal B_\sigma(\bo)}\ \min_{q\in Q(b)}V^\pi_p(\brho)V^\pi_q(\brho).
\]
\end{proposition}
Indeed, if $q\in\Pi_\sigma^p$, then setting $b_{sa}:=\kl_{sa}(p,q)$ gives $b\in\mathcal B_\sigma(\bo)$ and $q\in Q(b)$.
Conversely, if $q\in Q(b)$ for some $b\in\mathcal B_\sigma(\bo)$, then $\sum_{s,a}\omega_{sa}\kl_{sa}(p,q)\le \sum_{s,a}\omega_{sa}b_{sa}\le \sigma$, so $q\in\Pi_\sigma^p$. The proof of Proposition~\ref{prop:product_box_reduction} is deferred to Appendix~\ref{subsec:product_box_reduction}.

Thus, solving \eqref{opt:no sp} can be decomposed into an outer search over $b\in\mathcal B_\sigma(\bo)$ and an inner exact projected-gradient method over the product set $Q(b)$.
The inner projection factorizes across state--action pairs:
\[
\proj_{Q(b)}(x)
=
\Bigl(\proj_{\{q\in\Delta(\mathcal S):\kl_{sa}(p,q)\le b_{sa}\}}(x_{sa})\Bigr)_{(s,a)\in\mathcal S\times\mathcal A},
\]
so each block projection reduces to a one-dimensional dual root-finding problem on a KL-ball-constrained simplex.

\begin{algorithm}
    \caption{Nested Projected Gradient Descent}\label{alg:constrained_pg}
    \begin{algorithmic}[1]
    \State {\bf Input:} $(p,\zeta,\sigma,\bo)$.
    \If{$V_{p}^\pi(\brho)=0$}
        \State {\bf Output:} $0$
    \EndIf
    \State Define $\bar{\mathcal{S}}, \bar{\mathcal{A}}, \bar{s}, \bar{a}, \bar{r}, \bar{p}$, and $\bar{V}^{\bar{\pi}}_{\bar{p}}(\bar{\brho})$ as in Section~\ref{subsec:reversed MDP}
    \State $r_{\max}\gets \max_{s,a}|r(s,a)|$
    \State $L\gets 2(\gamma|\bar{\mathcal A}|+1)r_{\max}/(1-\gamma)^3$
    \State $h_\zeta\gets \zeta^2(1-\gamma)^4/\bigl(18|\mathcal S|^2|\mathcal A|^2|V^\pi_p(\brho)|^2r_{\max}^2\bigr)$
    \State Define $\mathcal G_{h_\zeta}:=\{b\in\mathcal B_\sigma(\bo): b_{sa}\in h_\zeta\mathbb Z_+,\,\forall (s,a)\}$
    \State $M\gets \left\lceil
      \frac{384(\gamma|\bar{\calA}|+1)|\bar{\calS}|r_{\max}|V^\pi_p(\brho)|\left\|1/\bar{\brho}\right\|_{\infty}^2}
      {(1-\gamma)^5\zeta}
      \right\rceil$
    \For{each $b\in\mathcal G_{h_\zeta}$}
        \State $\bar{\pi}^{(0)}_b \gets p$
        \For{$k =0, 1, \ldots, M-1$}
            \State $\bar{\pi}^{(k+1)}_b \gets$
            \Statex \hspace{\algorithmicindent}$\proj_{Q(b)} \left(\bar{\pi}^{(k)}_b - \frac{\operatorname{sign}(V_{p}^\pi(\brho))}{L}\nabla_{\bar{\pi}} \bar{V}^{\bar{\pi}^{(k)}_b}_{\bar{p}}(\bar{\brho}) \right)$
        \EndFor
        \State $u_b \gets V_{p}^\pi(\brho)\bar{V}^{\bar{\pi}^{(M)}_b}_{\bar{p}}(\bar{\brho})$
    \EndFor
    \State {\bf Output:} $u_{\zeta}\gets \min_{b\in \mathcal G_{h_\zeta}}u_b$
    \end{algorithmic}
\end{algorithm}

Line 8 fixes the grid resolution explicitly. By Lemma~\ref{lem:product_box_mesh}, this choice guarantees that coordinatewise flooring of an optimal budget vector to the grid $\mathcal G_{h_\zeta}$ incurs an outer discretization error of at most $\zeta/3$.

The following theorem provides the fixed-budget inner convergence guarantee together with the global approximation guarantee of the nested projected-gradient descent method.
\begin{theorem}\label{thm:proj_pg_rate}
Under Assumptions~\ref{apt:general} and~\ref{apt:sampling}, fix any $b\in\mathcal B_\sigma(\bo)$.
Let $(\bar\pi_b^{(k)})_{k\ge 0}$ be the exact projected-gradient sequence generated by the inner loop of Algorithm~\ref{alg:constrained_pg} on the product set $Q(b)$, and define
\[
\Delta_b^{(k)}
:=
\bar{V}^{\bar{\pi}^{(k)}_b}_{\bar{p}}(\bar{\brho})V^\pi_p(\brho)
-
\min_{q\in Q(b)}V^\pi_p(\brho)V^\pi_q(\brho).
\]
Then, for all $k\ge 1$,
\[
\Delta_b^{(k)}
\le
\frac{128(\gamma|\bar{\calA}|+1)|\bar{\calS}|r_{\max}|V^\pi_p(\brho)|}{(1-\gamma)^5 k}\bigl\|1/\bar{\brho}\bigr\|_{\infty}^2.
\]
In particular, if Algorithm~\ref{alg:constrained_pg} is run with the values of $h_\zeta$ and $M$ defined above, then its output $u_\zeta$ satisfies
\[
u_{\zeta}\ \ge\ \uno(\sigma,\bo,p)\ \ge\ u_{\zeta}-\zeta.
\]
\end{theorem}
The proof of Theorem~\ref{thm:proj_pg_rate} is provided in Appendix~\ref{subsec:proof_proj_pg_rate}. The key point is that, after the reduction of Proposition~\ref{prop:product_box_reduction}, each inner problem is solved over a genuine product set, so the projected-gradient analysis becomes exact and no slack mixing is needed.

Finally, when Algorithm~\ref{alg:constrained_pg} is used inside PTST (Algorithm~\ref{alg:static_algorithm}), we replace the unknown transition kernel $p$ with its empirical estimator $\hat{p}_t$, the tolerance $\zeta$ with $\zeta_t$, the allocation $\bo$ with $\hat{\bo}(t)$, and the radius $\sigma$ with the threshold $\beta(t,\delta)/t$.

\begin{remark}[Tradeoff of the product-box decomposition]
The exact reduction to product-box budgets transfers the global coupling induced by the KL constraint to an outer optimization over the local budget vector $b$.
For each fixed $b$, the feasible set becomes a genuine product set, so the inner projected-gradient method admits an exact analysis on each product box, yielding the approximation guarantee of Theorem~\ref{thm:proj_pg_rate}.
\end{remark}

\subsection{Optimality of PTST}

The following theorem, proved in Appendix~\ref{app:sample}, establishes the asymptotic optimality of the PTST algorithm.

\begin{theorem}\label{thm:sample}
   Suppose Assumptions~\ref{apt:general} and \ref{apt:sampling} hold.
   For any positive sequence $\{\zeta_t\}_{t=1}^\infty$ with $\lim_{t \to \infty}\zeta_t=0$, Algorithm~\ref{alg:static_algorithm} satisfies that $\mathbb{P}_p[\hat{\imath}\neq \ans (p)]\le \delta$, and
\[
\limsup_{\delta\to 0}\frac{\mathbb{E}_p[\tau]}{\log (1/\delta)}\le\left(\inf_{q\in \alt (p)}\sum_{s,a}\omega_{sa}\kl_{sa}(p,q)\right)^{-1}.
\]
\end{theorem}
The proof of the theorem relies on combining concentration results with a sensitivity analysis of $\uno$. We outline the main ideas of the proof below.

\begin{proof}[Proof Sketch of Theorem~\ref{thm:sample}]
First, leveraging the concentration inequalities and the fact that PTST tracks a fixed allocation $\bo\in \Sigma$, we can define, for a round $T$, a "good" event $\mathcal{C}_T(\xi)$ under which empirical estimates $\{\hat{p}_t\}_{t\ge \sqrt{T}}^T$ (resp. empirical allocation $\{\hat{\bo}(t)\}_{t\ge \sqrt{T}}^T$) are very close to $p$ (resp. $\bo$) and such that $\mathbb{P}_p[\mathcal{C}_T(\xi)^c]<\infty$ for large enough $T$.

Next, we can show that under event $\mathcal{C}_T(\xi)$, the "ideal" stopping rule (\ref{eq:ideal stopping}) will be activated when $\snc(0,\bo,p)\approx \beta(T,\delta)/T$. Our approximate stopping rule is more conservative. To measure its conservativeness, we conduct a sensitivity analysis of $\uno$: we prove that $c(\sigma_2-\sigma_1)\le \uno (\sigma_1,\bo,p)-\uno(\sigma_2,\bo,p)$ for some $c>0$ (Theorem~\ref{thm:sensitivity_u} in Appendix~\ref{app:sensitivity}) by using Lemma~\ref{lem:from apt}. This result is a consequence of a series of theorems in parameterized optimization and real analysis.

We finally establish that if $\sigma_2-\sigma_1\ge \zeta_T/c$ with $\sigma_1=\beta(T,\delta)/T$ and $\sigma_2=\snc(0,\bo,p)$, then Proposition~\ref{prop:bijection in main} implies that $\zeta_T\le \uno (\beta(T,\delta)/T)\le u_{\zeta_T}$. Thus, PTST stops when $\snc(0,\bo,p)\approx\beta(T,\delta)/T+\zeta_T/c\approx\log (1/\delta)T$. Or equivalently $T\approx \snc(0,\bo,p)^{-1}\log (1/\delta)\approx T^\star_{\bo}(p)\log (1/\delta)$, which completes the proof.
\end{proof}

\section{EXPERIMENTS}

In this section, we evaluate the proposed method in several settings. To the best of our knowledge, there is no prior algorithm tailored to policy testing in infinite-horizon discounted MDPs. Consequently, directly applying methods for thresholding bandits or other pure exploration problems is not straightforward. Instead, we adapt a best policy identification method to serve as a baseline: the KLB-TS algorithm of \cite{al2021adaptive}. To align KLB-TS with the policy testing objective, we consider two policies: one identical to $\pi$ and another, $\pi'$, satisfying $V^{\pi'}_{p}(\brho)=0$. We fix the sampling rule to be uniform over all state--action pairs. The stopping rule of KLB-TS is based on an upper bound obtained by convexifying the original minimax optimization problem. For the empirical study, we use a heuristic variant of PTST whose inner optimization directly projects onto $\bar{\Pi}_\sigma^{\,p}$ via SLSQP, rather than the nested product-box solver analyzed in Section~\ref{subsec:rates_policy_grad_const_MDP}.

We conduct experiments using three MDP settings: $|\mathcal{S}|=|\mathcal{A}|=2$, $|\mathcal{S}|=|\mathcal{A}|=3$, and $|\mathcal{S}|=|\mathcal{A}|=5$. In all cases, the discount factor is set to $\gamma=0.9$ and the initial state distribution is uniform over all states. The reward function $r(s,a)$, the transition kernel $p(\cdot\mid s,a)$, and the policy $\pi$ are specified in Tables~\ref{tab:all_mdp_info_2d}, \ref{tab:all_mdp_info}, and~\ref{tab:all_mdp_info_5} in Appendix~\ref{app:exp_details} for the respective settings.
For each setting, we vary $\delta$ from $10^{-15}$ to $10^{-2}$. The results are shown in Figure~\ref{fig:stopping_times_comparison}. In all three cases, PTST outperforms the baseline for all values of $\delta$. For further details, please refer to Appendix~\ref{app:exp_details}.

\begin{figure}[ht]%
    \centering
    \includegraphics[width=0.4\textwidth]{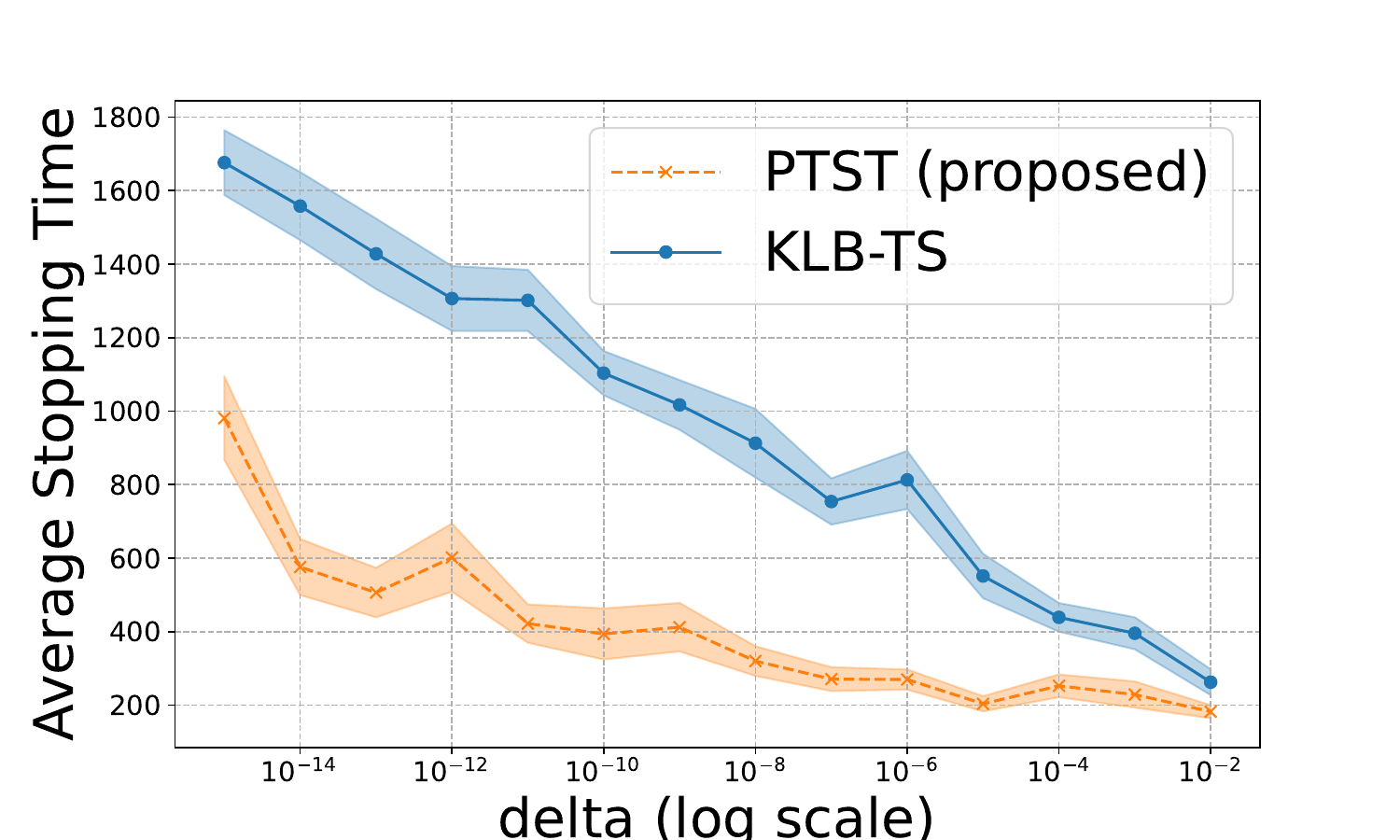}
    \includegraphics[width=0.4\textwidth]{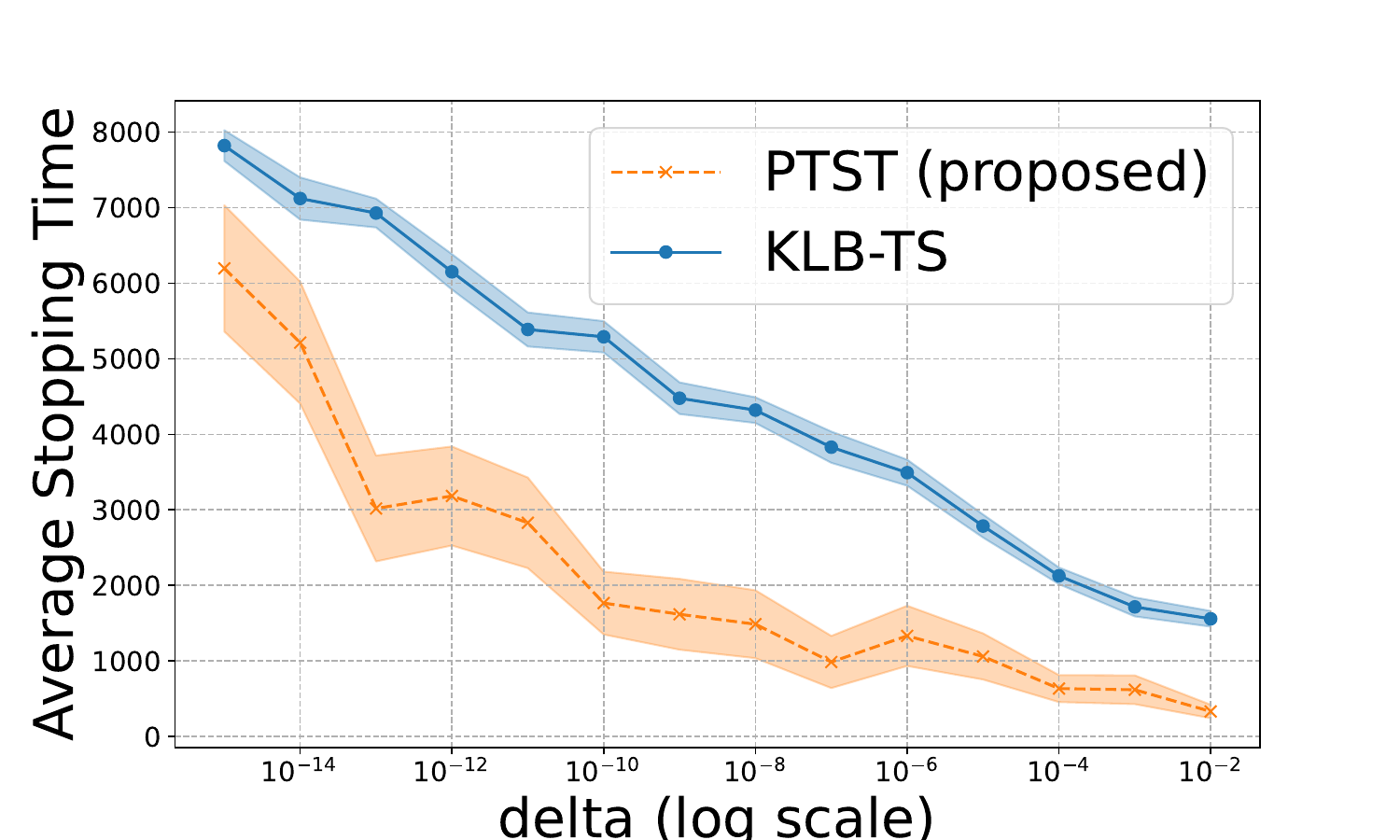}
    \includegraphics[width=0.4\textwidth]{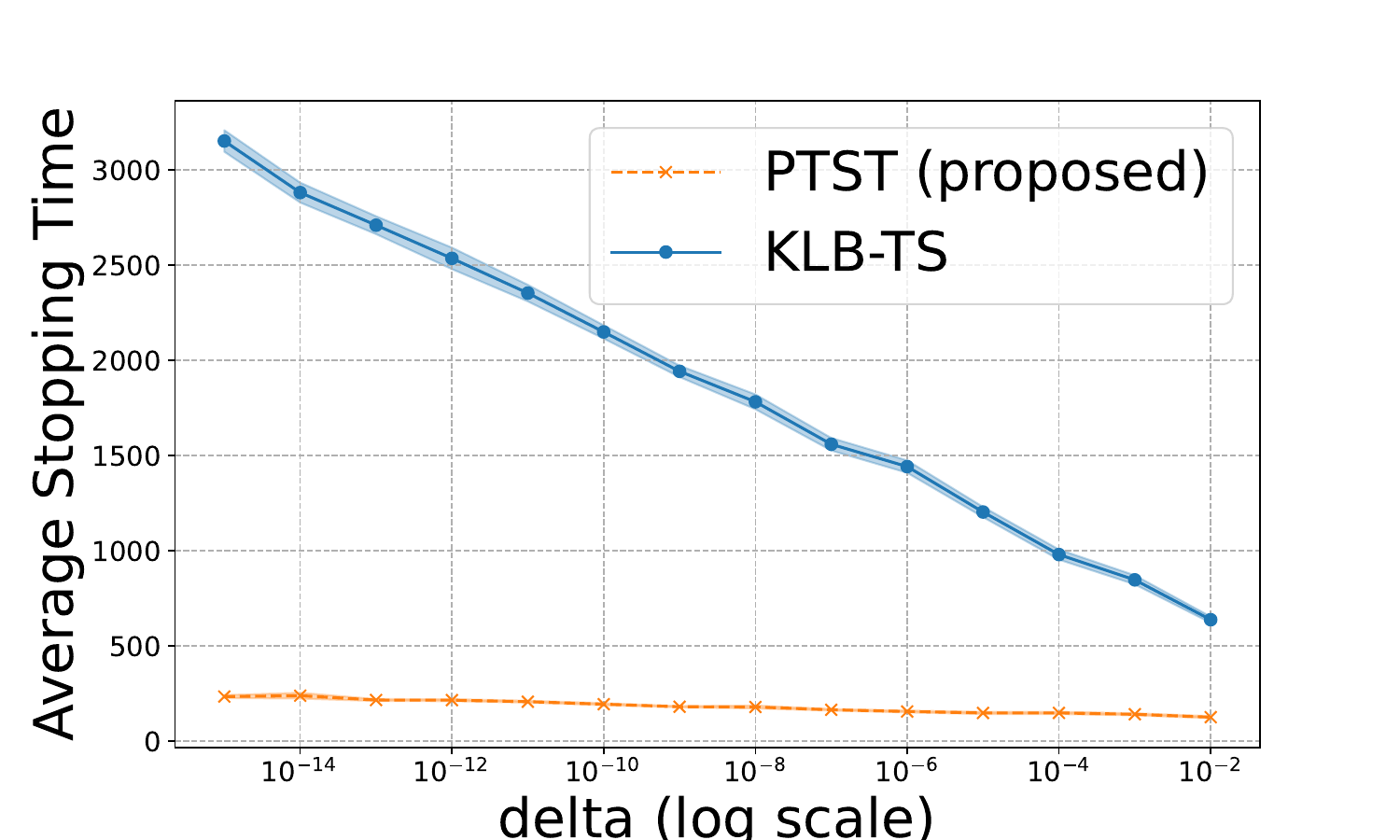}
    \caption{
        Comparison of average stopping times and $\delta$ for the proposed algorithm and KLB-TS.
        The top, middle, and bottom panels correspond to $|\mathcal{S}| = |\mathcal{A}| = 2, 3, 5$, respectively.
        Results are averaged over 30 instances. Error bars indicate the standard error.
    }
    \label{fig:stopping_times_comparison}
    \vspace{-0.5cm}
\end{figure}

 \section{LIMITATIONS AND DISCUSSION}

In this paper, we formulated the policy testing problem in discounted, tabular MDPs and characterized its instance-specific complexity under access to a generative model with a static sampling allocation. To the best of our knowledge, this is the first computationally tractable algorithm that achieves instance optimality for pure exploration in MDPs. The key ingredients are a reformulation that converts a non-convex lower-bound problem into a tractable program, an exact decomposition into product-box-constrained subproblems, and projected policy gradient methods in a reversed MDP.

Although our focus is policy testing, we expect the approach to extend to other pure exploration tasks, such as policy evaluation. We outline the required assumptions and corresponding optimization problems in Appendix~\ref{app:extension}.

Our current guarantees are restricted to the static sampling setting. The stopping rule based on the reversed MDP applies directly under static allocations, but extending the full framework, including the sampling design and the optimality guarantees, to adaptive sampling is an interesting direction for future work. Moreover, while Theorem~\ref{thm:proj_pg_rate} analyzes the nested product-box solver, the empirical study uses the simpler direct-projection heuristic described in Appendix~\ref{app:exp_details}; providing guarantees for this heuristic is left for future work.

Overall, we hope these findings provide a foundation for developing more efficient algorithms for pure exploration in MDPs.

\subsubsection*{Acknowledgements}

Kaito Ariu is supported by JSPS KAKENHI Grant Number 25K21291. Po-An Wang is supported by NSTC Grant Number 114-2118-M-007 -002 -MY2. A. Proutiere research is supported by Vetenskapr{\aa}det, Digital Futures, and the Wallenberg AI, Autonomous Systems and Software program.

\bibliography{ref}

\clearpage
\newpage
\section*{Checklist}

\begin{enumerate}

  \item For all models and algorithms presented, check if you include:
  \begin{enumerate}
    \item A clear description of the mathematical setting, assumptions, algorithm, and/or model. [Yes]
    \item An analysis of the properties and complexity (time, space, sample size) of any algorithm. [Yes]
    \item (Optional) Anonymized source code, with specification of all dependencies, including external libraries. [Yes]
  \end{enumerate}

  \item For any theoretical claim, check if you include:
  \begin{enumerate}
    \item Statements of the full set of assumptions of all theoretical results. [Yes]
    \item Complete proofs of all theoretical results. [Yes]
    \item Clear explanations of any assumptions. [Yes]     
  \end{enumerate}

  \item For all figures and tables 
  present empirical results, check if you include:
  \begin{enumerate}
    \item The code, data, and instructions needed to reproduce the main experimental results (either in the supplemental material or as a URL). [Yes]
    \item All the training details (e.g., data splits, hyperparameters, how they were chosen). [Yes]
    \item A clear definition of the specific measure or statistics and error bars (e.g., with respect to the random seed after running experiments multiple times). [Yes]
    \item A description of the computing infrastructure used. (e.g., type of GPUs, internal cluster, or cloud provider). [Yes]
  \end{enumerate}

  \item If you are using existing assets (e.g., code, data, models) or curating/releasing new assets, check if you include:
  \begin{enumerate}
    \item Citations of the creator If your work uses existing assets. [Not Applicable]
    \item The license information of the assets, if applicable. [Not Applicable]
    \item New assets either in the supplemental material or as a URL, if applicable. [Not Applicable]
    \item Information about consent from data providers/curators. [Not Applicable]
    \item Discussion of sensible content if applicable, e.g., personally identifiable information or offensive content. [Not Applicable]
  \end{enumerate}

  \item If you used crowdsourcing or conducted research with human subjects, check if you include:
  \begin{enumerate}
    \item The full text of instructions given to participants and screenshots. [Not Applicable]
    \item Descriptions of potential participant risks, with links to Institutional Review Board (IRB) approvals if applicable. [Not Applicable]
    \item The estimated hourly wage paid to participants and the total amount spent on participant compensation. [Not Applicable]
  \end{enumerate}

\end{enumerate}

\newpage
\appendix
\thispagestyle{empty}

\onecolumn
\aistatstitle{Policy Testing in Markov Decision Processes \\
Supplementary Materials}
\section{Additional Related Work}\label{app:further_relatedwork}

\paragraph{Comparison with Thresholding Bandits.} When compared with the bandit literature, the policy testing problem can be regarded as a generalization of the thresholding bandit problem \citep{chen2015optimal, locatelli2016optimal, degenne2019pure,wang2021fast}, where the goal is to adaptively sample each arm and identify those whose mean reward exceeds a given threshold; our work generalizes this concept to the value function setting in MDPs. Similar settings using known thresholds in MABs have also led to other important variants, such as the good arm identification \citep{kano2019good} and bad arm existence checking problems \citep{tabata2020bad}. Such a setting has applications in scenarios with practical constraints on the horizon, such as in timely recommendations~\citep{kano2019good}. It can also be viewed as an application of the concept of a satisficing objective \citep{russo2022satisficing, reverdy2016satisficing}.

\paragraph{Analysis of Policy Gradient Methods.} Our inner projected-gradient analysis (Section~\ref{subsec:rates_policy_grad_const_MDP}) is partly based on recent advances in the analysis of convergence rates for policy gradient methods \cite{Xiao2022convergence, agarwal2021theory}. While studies on the convergence rate of policy gradients with constraints exist \citep{ding2023last, montenegro2024last}, most focus on linear constraints on the state visitation distribution. In contrast, our results work with nonlinear KL constraints by decomposing the global feasible set into product-box subproblems, and combine the resulting inner projected-gradient guarantees with an outer budget search and a sensitivity analysis (see Sections~\ref{sec:testing alg} and \ref{app:sensitivity}).

\paragraph{Policy Evaluation.} Furthermore, our technique suggests that the instance-specific optimality is likely to be achievable for policy evaluation problems, including off-policy evaluation \citep{uehara2022review} and more general formulations \citep{russo2025adaptive}. We discuss extensions to policy evaluation in Section~\ref{app:extension}.

\section{Extensions Toward Other Pure Exploration Tasks}\label{app:extension}

\paragraph{Policy Evaluation.}
Policy evaluation is the task where we aim to approximate the value of a given policy up to a predetermined constant $\varepsilon$ with a certain confidence. Specifically, if $\hat{v}$ denotes the approximation of $V^\pi_p(\brho)$, the goal is to minimize the number of samples $\mathbb{E}_p[\tau]$ while satisfying $\mathbb{P}_p[\left|\hat{v}-V^\pi_{p}(\brho)\right|>\varepsilon]<\delta$. Similar to Assumption~\ref{apt:general} considered for the policy testing, we present Assumption~\ref{apt:eva} for policy evaluation. Notice that if Assumption~\ref{apt:eva} does not hold, returning $\hat{v}$ as an arbitrary value between $r^\pi(\brho)+\min_{s}\frac{\gamma}{1-\gamma}r^{\pi}(s)$ and $r^\pi(\brho)+\max_{s}\frac{\gamma}{1-\gamma}r^{\pi}(s)$ satisfies that $\left|\hat{v}-V^\pi_{p}(\brho)\right|\le \varepsilon$.
\begin{assumption}\label{apt:eva}
$\rho_s > 0$ for all $s \in {\cal S}$. 
    $r$ and $\brho$ satisfy:
   $$
   \max_{s}\frac{\gamma}{1-\gamma}r^{\pi}(s)-\min_{s}\frac{\gamma}{1-\gamma}r^{\pi}(s)>\varepsilon.
   $$
\end{assumption}
As discussed in Section~\ref{sec:testing alg} and \ref{subsec:nc and no}, solving the stopping condition boils down to identify whether the minimal value of the following two optimization problems is larger than $\beta(t,\delta)/t$. 
\begin{align}\label{opt:nc-eva-1}
   \inf_{q}  \sum_{s,a} \omega_{sa}\kl_{sa}(p,q) & \quad \text{s.t. }
   V^\pi_q(\brho)- V^\pi_p(\brho)+\varepsilon<0,%
\end{align}
and 
\begin{align}\label{opt:nc-eva-2}
   \inf_{q}  \sum_{s,a} \omega_{sa}\kl_{sa}(p,q) & \quad \text{s.t. }
V^\pi_p(\brho)-   V^\pi_q(\brho)+\varepsilon<0.%
\end{align}
Under Assumption~\ref{apt:eva}, either $\{q\in\mathcal{P}:V^\pi_q(\brho)- V^\pi_p(\brho)+\varepsilon<0\}$ or $\{q\in \mathcal{P}: V^\pi_p(\brho)-   V^\pi_q(\brho)+\varepsilon<0\}$ will be nonempty. For simplicity, we restrict our attention to solving (\ref{opt:nc-eva-2}) and assume $\{q\in \mathcal{P}: V^\pi_p(\brho)-   V^\pi_q(\brho)+\varepsilon<0\}\neq \emptyset$. In the following, we prove Assumption~\ref{ass:hg_new} holds with the corresponding substitution in Lemma~\ref{lem:eva}, then one can implement a projected policy gradient method to approximate the value of its dual problem, as described in Section~\ref{subsec:rates_policy_grad_const_MDP}. 
\begin{align}
    \min_{q} V^\pi_p(\brho)-   V^\pi_q(\brho)+\varepsilon&\quad \hbox{s.t.} \sum_{sa}\omega_{sa}\kl_{sa}(p,q)\le \frac{\beta(t,\delta)}{t}.
\end{align}
\begin{lemma}\label{lem:eva}
    When $\{q\in \mathcal{P}: V^\pi_p(\brho)-   V^\pi_q(\brho)+\varepsilon<0\}\neq \emptyset$, Assumption~\ref{ass:hg_new} holds with the following substitution.
    \begin{align*}
&   \mathcal{X} = {\cal P}, &  x=q,\\
&h(q) = \sum_{s,a} \omega_{sa}\kl_{sa}(p,q), & g(q) = V^\pi_p(\brho)-   V^\pi_q(\brho)+\varepsilon.
\end{align*}
\end{lemma}
\begin{proof}
(a) Let $\underline{x}=p$, one has that $h(\underline{x})=\sum_{sa}w_{sa}\kl_{sa}(p,p)=0$.\\
\noindent
(b) is a direct consequence of the assumption that $\{q\in \mathcal{P}: V^\pi_p(\brho)-   V^\pi_q(\brho)+\varepsilon<0\}\neq \emptyset$.
\\
\noindent
(c) and (d) hold as the proof in Proposition~\ref{prop:bijection in main}.
\end{proof}

\paragraph{Best Policy Identification.}
Here we discuss the extension of our methods to best policy identification. As shown in \cite{al2021adaptive}, $\text{Alt}(p) := \bigcup_{s \in \mathcal{S}} \bigcup_{a \neq \pi(s)} \{ q \in \mathcal{P} : Q_q^{\pi}(s, a) > V_q^{\pi}(s) \}$. The nonconvex optimization we are interested in is
\[
\inf_{q \in \text{Alt}(p)} \sum_{s, a}\omega_{sa} \text{KL}_{sa}(p, q) = \min_{s \in \mathcal{S}} \min_{a \neq \pi(s)} \inf_{q : Q_q^{\pi}(s, a) > V_q^{\pi}(s)} \sum_{s, a}\omega_{sa} \text{KL}_{sa}(p, q).
\]
For a fixed pair $(s, a)$ such that $s \in \mathcal{S}, a \neq \pi(s)$, we define (NC-u) as:
\[
\inf_{q \in \mathcal{P}} \sum_{s, a} \omega_{sa}\text{KL}_{sa}(p, q)\quad\text{s.t. }V_q^{\pi}(s) - Q_q^{\pi}(s, a) < u.
\]

As described in Section~\ref{subsec:nc and no}, (NC-u) can be transformed into (NO-$\sigma$) written below.
\[
\min_{q \in \mathcal{P}} V_q^{\pi}(s) - Q_q^{\pi}(s, a)\quad \text{s.t. }\sum_{s, a} \omega_{sa}\text{KL}_{sa}(p, q) \leq \sigma.
\]
Since $V_q^{\pi}(s) = Q_q^{\pi}(s, \pi(s))$, the objective function of (NO-$\sigma$) becomes $Q_q^{\pi}(s, \pi(s)) - Q_q^{\pi}(s, a)$.

Leveraging the construction of the reversed MDP in Section~\ref{subsec:reversed MDP}, we obtain
\[
Q_q^{\pi}(s, \pi(s)) - Q_q^{\pi}(s, a) = \bar{V}_{\bar{p}}^{\pi}(s) - \bar{V}_{\bar{p}}^{\pi}(s'),
\]
where $\bar{s} = (s, \pi(s))$ and $\bar{s}' = (s, a)$. The objective function is now the difference between the value functions of two states. Optimizing such objective functions may not be possible using the standard policy gradient method (as in our paper). However, we expect that this issue could be resolved by employing a more sophisticated algorithm design and analysis. This would then yield an instance-optimal algorithm for the best policy identification setting.

\paragraph{Toward the Forward Model.}
Extending our results to the forward (online, single-trajectory) interaction model—in which the next state depends on the current state and action—remains an important direction for future work, likely building on techniques from \cite{al2021navigating}. Such extensions may require assuming that all state–action pairs with positive probability under the policy are visited. Most algorithms in this literature assume that the transition kernel induces a communicating MDP, i.e., every state is reachable from every other state under some deterministic policy \citep{al2021navigating,russo2024multi,taupin2023best}. However, this assumption may be unrealistic in practice, particularly when the presence of transient states is unknown. Recent work by \citet{tuynman2024finding} weakens this requirement by assuming that the MDPs are weakly communicating, which is still stronger than Assumption~\ref{apt:general}.

While this extension is important, it is largely orthogonal to our main contribution, which establishes instance-specific optimality in the generative-model setting with computational efficiency. We therefore view our results as a step toward addressing more general and realistic settings.

\section{Instance-Specific Sample Complexity Lower Bound--Proof of Theorem~\ref{thm:stat_lower}}\label{app:lower}
\begin{proof}[Proof of Theorem~\ref{thm:stat_lower}]
Consider two cases, (i) $\inf_{q\in \alt (p)}\sum_{s,a}\omega_{sa}\kl_{sa}(p,q)=\infty$; (ii) $\inf_{q\in \alt (p)}\sum_{s,a}\omega_{sa}\kl_{sa}(p,q)<\infty$. Observe that our theorem holds directly in case (i), one only needs to focus on the case (ii). By our Lemma~\ref{lem:finite kl}, one can derive  $\tilde{q}\in \alt(p)$ such that $\kl_{sa}(p,\tilde{q})<\infty$ for all $s\in \mathcal{S},a\in\mathcal{A}$

    Let $q\in \alt (p)$, which is a nonempty set thanks to Assumption~\ref{apt:general}. Let $\PP_{p}$ and $\PP_{q}$ denote the probability measure generated by $p$ and $q$ respectively. According to property (i) of the $\delta$-PC algorithm definition, the stopping time $\tau$ is almost surely finite. Using Lemma~1 in \cite{al2021adaptive} and the classical data processing inequality (see e.g. Lemma~1 in \cite{kaufmann2016complexity}), we derive that for any ${\cal F}_\tau$-measurable event $E$,
    \begin{equation}\label{eq:lb1}
    \sum_{s,a}\EE_p[N_{sa}(\tau)]\kl_{sa}(p,q)\ge \skl (\PP_{p}[E],\PP_{q}[E])  , 
    \end{equation}
    where $\skl(a,b)$ denotes the Kullback-Leibler (KL) divergence between two Bernoulli distributions with means $a$ and $b$.
     With choice $E=\{\hat{\imath}=\ans (p)\}$, the definition of $\delta$ -PC algorithm (Definition~\ref{def:delta-PC}) and the assumption that $q\in \alt (p)$ yield that $\PP_p[E]\ge 1-\delta$ and $\PP_q[E]\le \delta$. After applying the monotonicity of KL divergence, we obtain $ \skl (\PP_{p}[E],\PP_{q}[E])  \ge \skl(\delta,1-\delta)$. Thanks to the assumption on sampling rule, for any $\varepsilon$, one can find $c_{\varepsilon}>0$ such that $\mathbb{E}_{p}[N_{sa}(t)]\le t(\omega_{sa}+\varepsilon)+c_{\varepsilon},\,\forall s\in\mathcal{S},a\in \mathcal{A}$.
     As (\ref{eq:lb1}) holds for any $q\in \alt (p)$, 
    \begin{align}
   \nonumber     \skl(\delta,1-\delta)&\le \inf_{q\in \alt(p)}\EE_{p}[\tau]  \sum_{s,a}\frac{\EE_p[N_{sa}(\tau)]}{\EE_{p}[\tau]}\kl_{sa}(p,q)\\
  \label{eq:stat_LB-1}      &\le\EE_{p}[\tau]  \left(\inf_{q\in \alt (p)} \sum_{s,a}(\omega_{sa}+\varepsilon+\frac{c_\varepsilon}{\EE_{p}[\tau]})\kl_{sa}(p,q)\right)\\
\label{eq:stat_LB-2}        &\le \EE_{p}[\tau]  \left(\sum_{s,a}(\omega_{sa}+\varepsilon+c_\varepsilon)\kl_{sa}(p,\tilde{q})\right),
    \end{align}
    where the last inequality follows as $\EE_{p}[\tau]\ge 1$ and $\tilde{q}\in\alt(p).$ Since $\left(\sum_{s,a}(\omega_{sa}+\varepsilon+c_{\varepsilon})\kl_{sa}(p,\tilde{q})\right)$ is finite, we conclude $\EE_{p}[\tau]\to\infty$ as $\skl(\delta,1-\delta)\to \infty $ if $\delta\to 0$ from (\ref{eq:stat_LB-2}). 
 Rearranging (\ref{eq:stat_LB-1}) yields that 
   \begin{equation}\label{eq:stat_LB-3}
        \left(\inf_{q\in \alt (p)} \sum_{s,a}(\omega_{sa}+\varepsilon+\frac{c_\varepsilon}{\EE_{p}[\tau]})\kl_{sa}(p,q)\right)^{-1}\le \frac{\EE_{p}[\tau]}{\skl(\delta,1-\delta)}
   \end{equation}
 Using the fact that $ \skl(\delta,1-\delta)\approx \log (1/\delta)$ and $\EE_{p}[\tau]\to\infty$ when $\delta $ goes to zero, one can conclude the theorem by taking the limit inferior on both sides of (\ref{eq:stat_LB-3}) as $\varepsilon$ is taken arbitrarily.
\end{proof}

\begin{lemma}\label{lem:finite kl}
    If $\inf_{q\in\alt(p)}\sum_{s,a}\omega_{sa}\kl_{sa}(p,q)<\infty$, there is $\tilde{q}\in \alt(p)$ such that $\kl_{sa}(p,\tilde{q})<\infty,\,\forall s\in\mathcal{S},a\in\mathcal{A}.$  
\end{lemma}

\begin{proof}
As $\inf_{q\in\alt(p)}\sum_{s,a}\omega_{sa}\kl_{sa}(p,q)<\infty$, there is $q'\in \alt(p)$ (or equivalently $V^\pi_p(\brho)V^\pi_q(\brho)<0$) such that $\sum_{s,a}\omega_{sa}\kl_{sa}(p,q')<\infty$. It is done if $\omega_{sa}>0$ for all $s,a$ as one can take $\tilde{q}=q'$.\\
Suppose there are pairs $(s,a)$ such that $\omega_{sa}=0$, $\kl_{sa}(p,q')=\infty$. Observe that $\kl_{sa}(p,q')=\infty$ if and only if $\exists s'\in\mathcal{S}$ such that $q'(s'\mid s,a)=0$ but $p(s'\mid s,a)>0$. Since $V^\pi_p(\brho)V^\pi_q(\brho)$
   is a continuous mapping on $q$ by Lemma~\ref{lem:simulation}, one can hence obtain $\tilde{q}\in \alt(p)$ which satisfies $\tilde{q}(s'\mid s,a)>0$ for all $s',s,a$ by slightly perturbing $q'$. Thus, one has $\kl_{sa}(p,\tilde{q})<\infty,\,\forall s\in\mathcal{S},a\in\mathcal{A}.$  
\end{proof}

\section{Dual Interpretation of the Non-Convex Problems--Proof of Proposition~\ref{prop:bijection in main}}\label{app:dual}
Here, the optimization problems (\ref{opt:nc up}) and (\ref{opt:no sp}) are abstracted as the following two optimization problems, respectively.

\begin{align}\label{opt:NC-u}
  \tag{NC-$u$}  \inf_{x\in {\cal X}}h(x)\quad  \text{s.t. }g(x)< u,
\end{align}
and
\begin{align}\label{opt:NO-sigma}
  \tag{NO-$\sigma$}  \min_{x\in {\cal X}}g(x)   \quad   \text{s.t. }h(x)\le \sigma,
\end{align}
where $h,g:{\cal X}\to \RR\cup\{\infty\}$, ${\cal X}\subset \RR^d$ is a nonempty set, $d\in \NN$, and $u,\sigma\in\RR$. 
Let $\snc(u)$ and $\uno(\sigma)$ denote the value of (\ref{opt:NC-u}) and that of (\ref{opt:NO-sigma}), respectively.  As one can see, in Section~\ref{subsec:nc and no}, we make the substitutions $\mathcal{X}=\mathcal{P}$, $x=q$, $h(q)=\sum_{sa}\omega_{sa}\kl_{sa}(p,q)$, and $g(q)=V^\pi_{p}(\brho)V^\pi_q(\brho)$. As stated in Proposition~\ref{prop:bijection}, the desired bijection holds under the mild assumptions in Assumption~\ref{ass:hg_new}. In Appendix~\ref{subsec:prove_asm}, we verify that these assumptions hold for the aforementioned substitution.
\medskip
\begin{assumption}\label{ass:hg_new}
The following conditions hold.
    \begin{enumerate}[label=(\alph*)]
    \item $\exists \underline{x}\in \mathcal{X}$ such that $h(\underline{x})\le 0$%
     \item $\min_{x\in\mathcal{X}}g(x)<0$.
    \item All local minimums of $g$ ($h$ resp.) in $\mathcal{X}$ are global minimums of $g$ ($h$ resp.).
    \item $h,g$ are continuous mappings.
    \end{enumerate}
\end{assumption}

\medskip

The following proposition shows that, under Assumption~\ref{ass:hg_new}, there exists a bijection between problems \eqref{opt:NC-u} and \eqref{opt:NO-sigma}.

\begin{proposition}\label{prop:bijection}
    Under Assumption~\ref{ass:hg_new}, we have
\begin{align}
\label{eq:inverse-su}    
\hbox{for all } \sigma\ge 0 \hbox{ such that }\uno(\sigma)> \min_{x\in\mathcal{X}}g(x),&\quad \snc(\uno(\sigma))=\sigma,\\
\label{eq:inverse-us}    \hbox{for all } u\in (\min_{x\in\mathcal{X}}g(x),\uno(0)]&\quad \uno(\snc(u))=u.
\end{align}
\end{proposition}
In words, $\snc(u)$ and $\uno(\sigma)$ are one-to-one mappings as decreasing functions, and knowing the value of $\uno(\sigma)$ would be equivalent to knowing the value of $\snc(u)$.

\begin{proof}[Proof of Proposition~\ref{prop:bijection}]
    We first show (\ref{eq:inverse-su}). Let $\sigma\ge 0$ and $u=\uno(\sigma)>\min_{x\in\mathcal{X}}g(x)$. By Lemma~\ref{lem:up hg} and Lemma~\ref{lem:lower hg}, we get $\snc(u)\le \sigma$ and $\snc(u)\ge \sigma$ respectively, which directly yield (\ref{eq:inverse-su}).  \\
  For proving (\ref{eq:inverse-us}), we consider $u\in (\min_{x\in\mathcal{X}}g(x),\uno(0)]$. Using intermediate value theorem and the continuity of $\uno (\cdot)$ proved in Lemma~\ref{lem:uno cont} given in Appendix~\ref{subsec:uno cont}, we deduce that there is $\sigma \in [0,\infty)$ such that $\uno(\sigma)=u$. As a consequence, 
$$
\uno(\snc(u))=\uno(\snc(\uno(\sigma)))=\uno(\sigma)=u,
$$
where the second equation is the application of (\ref{eq:inverse-su}).

\end{proof}

\begin{lemma}\label{lem:up hg}
 Under Assumption~\ref{ass:hg_new}, whenever $\sigma\ge 0, u\in (\min_{x\in \mathcal{X}}g(x),\infty)$, and $\uno(\sigma)\le   u$, then $\snc(u)\le \sigma$ holds.
\end{lemma}
\begin{proof}[Proof of Lemma~\ref{lem:up hg}]
      As $\uno(\sigma)\le u$, $\exists x'\in \mathcal{X}$, such that $h(x')\le \sigma$, and $g(x')\le u$. If $g(x')<u$ (case i), then $x'$ is a feasible point for (\ref{opt:NC-u}), and therefore $\snc (u)\le h(x')\le \sigma$. We next consider the case where $g(x')=u$ (case ii). \\
       Since $g(x')=u>\min_{x\in\mathcal{X}}g(x)$, $x'$ is not a global minimum of $g$ in $\mathcal{X}$. Hence, by Assumption~\ref{ass:hg_new}-(c), we deduce that $x'$ is not a local minimum either. Thus, there is a sequence of $\{x_n\}_{n=1}^\infty$ such that $x_n\xrightarrow{n\to \infty} x'$ and $g(x_n)<u,\,\forall n$. As a consequence of Assumption~\ref{ass:hg_new}-(d), $\lim_{n\to\infty} h(x_n)=h(x')\le \sigma$, which yields $\snc (u)\le \sigma$. 
\end{proof}

\medskip
\begin{lemma}\label{lem:lower hg}
Whenever $\sigma\ge 0,u\in \RR$ and $\uno(\sigma)\ge  u$, $\snc(u)\ge \sigma$ holds.
\end{lemma}
\begin{proof}[Proof of Lemma~\ref{lem:lower hg}]
    Suppose in contrast, $\snc (u)<\sigma$. There exist $x\in {\cal X}$ such that $g(x)<u,\,h(x)<\sigma$. Hence $\uno(\sigma)<u$, which contradicts that $\uno(x)\ge u$.
\end{proof}

\subsection{Proof of Proposition~\ref{prop:bijection in main}}\label{subsec:prove_asm}

\begin{proof}[Proof of Proposition~\ref{prop:bijection in main}]
Thanks to Proposition~\ref{prop:bijection}, the proof is completed by verifying that the following substitutions satisfy Assumption~\ref{ass:hg_new} when $p\in \Ptest$. 
\begin{align*}
&   \mathcal{X} = {\cal P}, &  x=q,\\
&h(q) = \sum_{s,a} \omega_{sa}\kl_{sa}(p,q), & g(q) = V_q^\pi(\brho) V_p^\pi(\brho).
\end{align*}
\\
\noindent
(a) Let $\underline{x}=p$, one has that $h(\underline{x})=\sum_{sa}w_{sa}\kl_{sa}(p,p)=0$.\\
\noindent
(b) Due to Assumption~\ref{apt:general} and inequalities (\ref{eq:maxmin V}), there exists $q\in \mathcal{P}$ such that $\text{sign}(V^\pi_{q}(\brho))\neq \text{sign}(V^\pi_{p}(\brho))$, we have $\min_{x\in\mathcal{X}}g(x)\le V^\pi_{q}(\brho)V^\pi_{p}(\brho)<0$.
\\
\noindent
(c) The local minimum of $h$ is the global minimum since $h$ is a convex function. As for $g$, we reverse the policy and transition kernel as in Section~\ref{subsec:reversed MDP}. Since the reversed MDP is still a tabular MDP, Theorem 1 in \cite{bhandari2024global} has verified that all the stationary points for the value function on the policy space are global optima. As a consequence, the local minimum of $g$ is the global minimum.\\
\noindent
(d) It is clear that $h$ is a continuous function. The continuity of $g$ directly follows from the simulation lemma (Lemma~\ref{lem:simulation}).\\
\end{proof}

\section{Convergence Analysis of Nested Projected Gradient Descent--Proof of Theorem~\ref{thm:proj_pg_rate}}
\label{subsec:proof_proj_pg_rate}

\noindent\textbf{Notation.}
Throughout this appendix, we work with the sign-normalized reversed-MDP objective
\[
f(\bar{\pi})
:=
\operatorname{sign}(V^\pi_p(\brho))\,\bar{V}^{\bar{\pi}}_{\bar{p}}(\bar{\brho}),
\]
where $\bar p$ is the reversed-MDP transition kernel induced by the fixed target policy $\pi$ as defined in Section~\ref{subsec:reversed MDP}.
Recall from Section~\ref{subsec:rates_policy_grad_const_MDP} that
\[
Q(b)=\left\{\bar{\pi}\in\mathcal P:\kl_{sa}(p,\bar{\pi})\le b_{sa},\ \forall (s,a)\in\mathcal S\times\mathcal A\right\}.
\]
For any $b\in\mathcal B_\sigma(\bo)$, define
\[
f_b^\star:=\min_{\bar{\pi}\in Q(b)}f(\bar{\pi}).
\]

{\bf Roadmap.}
This appendix proves Theorem~\ref{thm:proj_pg_rate}.
We first analyze the inner projected-gradient method on a fixed
$b\in \mathcal{B}_{\sigma}(\bo)$ and establish its convergence rate
(Lemma~\ref{lem:product_box_rate}).
The proof follows the same general strategy as \citep{Xiao2022convergence}:
Lemma~\ref{lem:product_box_gmd} provides the weak gradient-domination property
that turns control of the gradient mapping into control of value suboptimality,
while Lemma~\ref{lem:smoothness_parameter} gives the smoothness estimate needed
for the projected-gradient descent argument.
We then combine this inner-loop guarantee with the discretization of the outer
search over $\mathcal{B}_{\sigma}(\bo)$; here
Lemma~\ref{lem:product_box_mesh} controls the error induced by the budget mesh.
Putting these ingredients together yields the proof of
Theorem~\ref{thm:proj_pg_rate}.

For readability, we defer two auxiliary components to separate subsections.
Appendix~\ref{subsec:proof_covergence_rate_general} recalls the general
projected-gradient convergence template used in the proof, and
Appendix~\ref{subsec:product_box_reduction} collects the product-box reduction
argument underlying the fixed-$b$ analysis.

\medskip

To derive the weak-gradient domination, we need an analogue of variational gradient domination (Lemma~4 in \citep{agarwal2021theory}) in a product-set.

\begin{lemma}[Product-set variational gradient domination]\label{lem:product_box_vgd}
Fix $b\in\mathcal B_\sigma(\bo)$ and let $\bar{\pi}_b^\star\in\arg\min_{\bar{\pi}\in Q(b)}f(\bar{\pi})$.
Then, for every $\bar{\pi}\in Q(b)$,
\[
f(\bar{\pi})-f_b^\star
\le
\left\|\frac{d^{\bar{\pi}_b^\star}_{\bar{p},\bar{\brho}}}{d^{\bar{\pi}}_{\bar{p},\bar{\brho}}}\right\|_\infty
\max_{\bar{\pi}'\in Q(b)}
\left\langle \nabla_{\bar{\pi}}f(\bar{\pi}),\,\bar{\pi}-\bar{\pi}'\right\rangle.
\]
In particular,
\[
f(\bar{\pi})-f_b^\star
\le
\frac{1}{1-\gamma}\left\|1/\bar{\brho}\right\|_\infty
\max_{\bar{\pi}'\in Q(b)}
\left\langle \nabla_{\bar{\pi}}f(\bar{\pi}),\,\bar{\pi}-\bar{\pi}'\right\rangle.
\]
\end{lemma}

\begin{proof}
Fix $\bar{\pi}\in Q(b)$ and let $\bar{\pi}_b^\star\in\arg\min_{\bar{\pi}'\in Q(b)}f(\bar{\pi}')$.
Write the reversed-MDP advantage as
\[
A_{\bar p}^{\bar\pi}((s,a),s')
:=
\bar Q_{\bar p}^{\bar\pi}((s,a),s')-\bar V_{\bar p}^{\bar\pi}(s,a).
\]
For each $(s,a)$, define
\[
g_b^{\bar\pi}(s,a)
:=
\max_{\substack{q(\cdot\mid s,a)\in\Delta(\mathcal S):\\ \kl_{sa}(p,q)\le b_{sa}}}\sum_{s'\in\mathcal S}q(s')\big(-A_{\bar p}^{\bar\pi}((s,a),s')\big).
\]
Since $\bar\pi_b^\star(\cdot\mid s,a)$ satisfies $\kl_{sa}(p,\bar\pi_b^\star)\le b_{sa}$, we have
\[
\sum_{s'}\bar\pi_b^\star(s'\mid s,a)\big(-A_{\bar p}^{\bar\pi}((s,a),s')\big)
\le
g_b^{\bar\pi}(s,a).
\]
By the performance-difference lemma on the reversed MDP (see Lemma~\ref{lem:simulation}),
\[
f(\bar\pi)-f_b^\star
=
\frac{1}{1-\gamma}
\sum_{s,a}d^{\bar\pi_b^\star}_{\bar p,\bar{\brho}}(s,a)
\sum_{s'}\bar\pi_b^\star(s'\mid s,a)\big(-A_{\bar p}^{\bar\pi}((s,a),s')\big)
\]
\[
\le
\frac{1}{1-\gamma}
\sum_{s,a}d^{\bar\pi_b^\star}_{\bar p,\bar{\brho}}(s,a)g_b^{\bar\pi}(s,a).
\]
Multiplying and dividing by $d^{\bar\pi}_{\bar p,\bar{\brho}}(s,a)$ yields
\[
f(\bar\pi)-f_b^\star
\le
\frac{1}{1-\gamma}
\left\|\frac{d^{\bar{\pi}_b^\star}_{\bar p,\bar{\brho}}}{d^{\bar\pi}_{\bar p,\bar{\brho}}}\right\|_\infty
\sum_{s,a}d^{\bar\pi}_{\bar p,\bar{\brho}}(s,a)g_b^{\bar\pi}(s,a).
\]
For each $(s,a)$, choose
\[
q^\star_{sa}\in\arg\max_{\substack{q(\cdot\mid s,a)\in\Delta(\mathcal S):\\ \kl_{sa}(p,q)\le b_{sa}}}\sum_{s'}q(s')(-A_{\bar p}^{\bar\pi}((s,a),s')),
\]
and define $\hat\pi\in Q(b)$ by $\hat\pi(\cdot\mid s,a)=q^\star_{sa}(\cdot)$.
Because $Q(b)$ is a direct product, $\hat\pi\in Q(b)$.
Using $\sum_{s'}\bar\pi(s'\mid s,a)A_{\bar p}^{\bar\pi}((s,a),s')=0$ for all $(s,a)$ and the directional policy-gradient identity from Lemma~\ref{lem:reversed_PG}, we obtain
\[
\sum_{s,a}d^{\bar\pi}_{\bar p,\bar{\brho}}(s,a)g_b^{\bar\pi}(s,a)
=
(1-\gamma)
\left\langle \nabla_{\bar\pi}f(\bar\pi),\,\bar\pi-\hat\pi\right\rangle
\]
\[
\le
(1-\gamma)
\max_{\bar\pi'\in Q(b)}
\left\langle \nabla_{\bar\pi}f(\bar\pi),\,\bar\pi-\bar\pi'\right\rangle.
\]
Substituting this bound yields the first claim.
The second one follows from $d^{\bar\pi}_{\bar p,\bar{\brho}}(s,a)\ge (1-\gamma)\bar\rho(s,a)$ for all $(s,a)$.
\end{proof}

\begin{lemma}[Weak gradient-mapping domination on a fixed product box]\label{lem:product_box_gmd}
Fix $b\in\mathcal B_\sigma(\bo)$ and assume that $f$ is $L$-smooth on $Q(b)$.
Define
\[
T_{L,b}(\bar\pi):=\proj_{Q(b)}\left(\bar\pi-\frac{1}{L}\nabla_{\bar\pi}f(\bar\pi)\right),
\qquad
G_{L,b}(\bar\pi):=L\bigl(\bar\pi-T_{L,b}(\bar\pi)\bigr).
\]
Then, for every $\bar\pi\in Q(b)$,
\[
f(T_{L,b}(\bar\pi))-f_b^\star
\le
\frac{2\sqrt{2|\bar{\mathcal S}|}}{1-\gamma}\left\|1/\bar{\brho}\right\|_\infty
\|G_{L,b}(\bar\pi)\|_2.
\]
\end{lemma}

\begin{proof}
Fix $\bar\pi\in Q(b)$ and let $T:=T_{L,b}(\bar\pi)$.
Applying Lemma~\ref{lem:product_box_vgd} at $T\in Q(b)$ gives
\[
f(T)-f_b^\star
\le
\frac{1}{1-\gamma}\left\|1/\bar{\brho}\right\|_\infty
\max_{\bar\pi'\in Q(b)}
\left\langle \nabla_{\bar\pi}f(T),\,T-\bar\pi'\right\rangle.
\]
Since $f$ is $L$-smooth and $T$ is the exact projected-gradient step of $\bar\pi$ on $Q(b)$, Theorem~1 of \citet{nesterov2013gradient} yields
\[
\left\langle \nabla_{\bar\pi}f(T),\,T-\bar\pi'\right\rangle
\le
2\|G_{L,b}(\bar\pi)\|_2\,\|T-\bar\pi'\|_2,
\qquad \forall \bar\pi'\in Q(b).
\]
Now $Q(b)\subseteq \mathcal P=\Delta(\bar{\mathcal A})^{\bar{\mathcal S}}$, and the Euclidean diameter of $\mathcal P$ is at most $\sqrt{2|\bar{\mathcal S}|}$.
Hence
\[
\max_{\bar\pi'\in Q(b)}\|T-\bar\pi'\|_2\le \sqrt{2|\bar{\mathcal S}|},
\]
which implies the claim.
\end{proof}

\begin{lemma}[Inner projected-gradient rate on a fixed product box]\label{lem:product_box_rate}
Fix $b\in\mathcal B_\sigma(\bo)$ and let $(\bar\pi_b^{(k)})_{k\ge 0}$ be the exact projected-gradient sequence on $Q(b)$:
\[
\bar\pi_b^{(k+1)}
=
\proj_{Q(b)}\left(\bar\pi_b^{(k)}-\frac{1}{L}\nabla_{\bar\pi}f(\bar\pi_b^{(k)})\right).
\]
Assume that $f$ is $L$-smooth on $Q(b)$.
Then, for every $k\ge 1$,
\[
f(\bar\pi_b^{(k)})-f_b^\star
\le
\max\left\{
\frac{64L|\bar{\mathcal S}|}{(1-\gamma)^2k}\left\|1/\bar{\brho}\right\|_\infty^2,
\left(\frac{\sqrt2}{2}\right)^k\bigl(f(\bar\pi_b^{(0)})-f_b^\star\bigr)
\right\}.
\]
\end{lemma}

\begin{proof}
By Lemma~\ref{lem:product_box_gmd}, the gradient-mapping domination condition holds on $Q(b)$ with
\[
\omega_b
:=
\frac{(1-\gamma)^2}{16|\bar{\mathcal S}|\,\|1/\bar{\brho}\|_\infty^2}.
\]
Applying Theorem~\ref{thm:convergence_rate_general} with this value of $\omega_b$ yields the claim.
\end{proof}

\begin{lemma}[Explicit grid discretization for product-box budgets]\label{lem:product_box_mesh}
Define
\[
\varphi(b):=\min_{q\in Q(b)}V_p^\pi(\brho)V_q^\pi(\brho).
\]
Let $b^\star\in\arg\min_{b\in\mathcal B_\sigma(\bo)}\varphi(b)$ and let $h>0$.
Define $\tilde b\in\mathcal B_\sigma(\bo)$ by coordinatewise flooring:
\[
\tilde b_{sa}:=h\left\lfloor b_{sa}^\star/h\right\rfloor,\qquad (s,a)\in\mathcal S\times\mathcal A.
\]
Then
\[
\varphi(\tilde b)-\varphi(b^\star)
\le
\frac{\sqrt{2}\,|\mathcal S||\mathcal A|\,|V_p^\pi(\brho)|\,r_{\max}}{(1-\gamma)^2}\sqrt{h}.
\]
In particular, if
\[
h_\zeta:=\frac{\zeta^2(1-\gamma)^4}{18\,|\mathcal S|^2|\mathcal A|^2\,|V_p^\pi(\brho)|^2\,r_{\max}^2},
\]
then
\[
\varphi(\tilde b)\le \varphi(b^\star)+\zeta/3.
\]
\end{lemma}

\begin{proof}
Choose $q^\star\in\Pi_\sigma^p$ such that
\[
V_p^\pi(\brho)V_{q^\star}^\pi(\brho)=\uno(\sigma,\bo,p),
\]
which exists because $\Pi_\sigma^p$ is compact and the objective is continuous.
Set
\[
b_{sa}^\star:=\kl_{sa}(p,q^\star),\qquad (s,a)\in\mathcal S\times\mathcal A.
\]
Then $b^\star\in\mathcal B_\sigma(\bo)$, $q^\star\in Q(b^\star)$, and Proposition~\ref{prop:product_box_reduction} gives
\[
\varphi(b^\star)=\uno(\sigma,\bo,p)=V_p^\pi(\brho)V_{q^\star}^\pi(\brho).
\]
For each $(s,a)$ with $b_{sa}^\star>0$, define
\[
\alpha_{sa}:=1-\tilde b_{sa}/b_{sa}^\star,
\qquad
\tilde q(\cdot\mid s,a):=(1-\alpha_{sa})q^\star(\cdot\mid s,a)+\alpha_{sa}p(\cdot\mid s,a).
\]
If $b_{sa}^\star=0$, then $q^\star(\cdot\mid s,a)=p(\cdot\mid s,a)$ and we set $\alpha_{sa}:=0$, so again $\tilde q(\cdot\mid s,a)=q^\star(\cdot\mid s,a)$.
By convexity of $q\mapsto \kl_{sa}(p,q)$ in its second argument,
\[
\kl_{sa}(p,\tilde q)
\le
(1-\alpha_{sa})\kl_{sa}(p,q^\star)
\le
(1-\alpha_{sa})b_{sa}^\star
=
\tilde b_{sa},
\]
so $\tilde q\in Q(\tilde b)$.

For each block, Pinsker's inequality and the identity $\alpha_{sa}b_{sa}^\star=b_{sa}^\star-\tilde b_{sa}\le h$ give
\[
\|\tilde q(\cdot\mid s,a)-q^\star(\cdot\mid s,a)\|_1
=
\alpha_{sa}\|p(\cdot\mid s,a)-q^\star(\cdot\mid s,a)\|_1
\le
\alpha_{sa}\sqrt{2\,\kl_{sa}(p,q^\star)}
\le
\sqrt{2h}.
\]
Summing over $(s,a)$ yields
\[
\|\tilde q-q^\star\|_1
\le
|\mathcal S||\mathcal A|\sqrt{2h}.
\]

Now let
\[
F(q):=V_p^\pi(\brho)V_q^\pi(\brho).
\]
By Lemma~\ref{lem:reversed_PG},
\[
\frac{\partial F(q)}{\partial q(s'\mid s,a)}
=
\frac{V_p^\pi(\brho)}{1-\gamma}d_{q,\brho}^\pi(s,a)\bigl(r(s,a)+\gamma V_q^\pi(s')\bigr).
\]
Since $d_{q,\brho}^\pi(s,a)\le 1$ and $|V_q^\pi(s')|\le r_{\max}/(1-\gamma)$, we obtain
\[
\left|\frac{\partial F(q)}{\partial q(s'\mid s,a)}\right|
\le
\frac{|V_p^\pi(\brho)|r_{\max}}{(1-\gamma)^2},
\qquad \forall q\in\mathcal P,\ \forall s,s',a.
\]
Therefore,
\[
|F(\tilde q)-F(q^\star)|
\le
\frac{|V_p^\pi(\brho)|r_{\max}}{(1-\gamma)^2}\|\tilde q-q^\star\|_1
\le
\frac{\sqrt{2}\,|\mathcal S||\mathcal A|\,|V_p^\pi(\brho)|\,r_{\max}}{(1-\gamma)^2}\sqrt{h}.
\]
Because $\tilde q\in Q(\tilde b)$,
\[
\varphi(\tilde b)-\varphi(b^\star)
\le
F(\tilde q)-F(q^\star)
\le
|F(\tilde q)-F(q^\star)|.
\]
This proves the first claim.
The second follows by substituting the displayed choice of $h_\zeta$.
\end{proof}

\begin{proof}[Proof of Theorem~\ref{thm:proj_pg_rate}]
If $V_p^\pi(\brho)=0$, then $\uno(\sigma,\bo,p)=0$ for every $\sigma$, and Algorithm~\ref{alg:constrained_pg} returns $0$.
Hence we may assume $V_p^\pi(\brho)\neq 0$.
Define the proof-only shorthand
\[
\varphi(b):=\min_{q\in Q(b)}V_p^\pi(\brho)V_q^\pi(\brho).
\]
\[
\Delta_b^{(k)}
:=
\bar{V}^{\bar{\pi}^{(k)}_b}_{\bar{p}}(\bar{\brho})V^\pi_p(\brho)-\varphi(b).
\]

\noindent\textit{Smoothness of the objective.}
Recall that
\[
f(\bar\pi)
:=
\operatorname{sign}(V^\pi_p(\brho))\,\bar V_{\bar p}^{\bar\pi}(\bar{\brho}).
\]
For each $\bar s=(s,a)\in \bar{\mathcal S}$, we have $\bar V_{\bar p}^{\bar\pi}(\bar s)=Q_q^\pi(s,a)$ under the identification $q=\bar\pi$, and
\[
\bar V_{\bar p}^{\bar\pi}(\bar{\brho})
=
\sum_{\bar s\in\bar{\mathcal S}}\bar\rho(\bar s)\bar V_{\bar p}^{\bar\pi}(\bar s).
\]
Hence Lemma~\ref{lem:smoothness_parameter} implies that, for any $\bar\pi,\tilde{\bar\pi}\in\mathcal P$,
\[
\begin{aligned}
\|\nabla_{\bar\pi}\bar V_{\bar p}^{\bar\pi}(\bar{\brho})
-\nabla_{\bar\pi}\bar V_{\bar p}^{\tilde{\bar\pi}}(\bar{\brho})\|_2
&\le
\sum_{\bar s\in\bar{\mathcal S}}\bar\rho(\bar s)
\|\nabla_{\bar\pi}\bar V_{\bar p}^{\bar\pi}(\bar s)
-\nabla_{\bar\pi}\bar V_{\bar p}^{\tilde{\bar\pi}}(\bar s)\|_2\\
&\le
\frac{2\gamma|\bar{\mathcal A}|r_{\max}}{(1-\gamma)^3}
\|\bar\pi-\tilde{\bar\pi}\|_2.
\end{aligned}
\]
Thus $\bar V_{\bar p}^{\bar\pi}(\bar{\brho})$ is $L_0$-smooth on $\mathcal P$ with
\[
L_0:=\frac{2\gamma|\bar{\mathcal A}|r_{\max}}{(1-\gamma)^3}.
\]
Multiplication by $\operatorname{sign}(V^\pi_p(\brho))$ does not change the smoothness constant, so $f$ is also $L_0$-smooth. In particular, any larger constant is admissible in the projected-gradient analysis. For the sequel, we use the slightly looser choice
\[
L:=\frac{2(\gamma|\bar{\mathcal A}|+1)r_{\max}}{(1-\gamma)^3},
\]
where the extra $+1$ is an inessential slack term used only to simplify the presentation.

\noindent\textit{Fixed-budget inner rate.}
Since $Q(b)\subseteq \mathcal P$, the objective $f$ is $L$-smooth on $Q(b)$. Applying Lemma~\ref{lem:product_box_rate} with this choice of $L$ and then multiplying the resulting inequality by $|V_p^\pi(\brho)|$, we obtain, for every $b\in\mathcal B_\sigma(\bo)$ and every $k\ge 1$,
\[
\Delta_b^{(k)}
\le
\max\left\{
\frac{128 (\gamma|\bar{\calA}|+1)|\bar{\calS}|r_{\max}|V^\pi_p(\brho)|}{(1-\gamma)^5 k}\left\|1/\bar{\brho}\right\|_{\infty}^2,
\frac{\Delta_b^{(0)}}{2^{k/2}}
\right\}.
\]
Since $\bar\pi_b^{(0)}=p$ and $|\bar V_{\bar p}^{\bar\pi}(\bar{\brho})|\le r_{\max}/(1-\gamma)$ for every $\bar\pi$, we have
\[
\Delta_b^{(0)}
\le
\frac{2|V_p^\pi(\brho)|r_{\max}}{1-\gamma}.
\]
Since $2^{-k/2}\le 2/k$ for all $k\ge 1$ and
\[
\frac{128 (\gamma|\bar{\calA}|+1)|\bar{\calS}|r_{\max}|V^\pi_p(\brho)|}{(1-\gamma)^5}\left\|1/\bar{\brho}\right\|_{\infty}^2
\ge
\frac{4|V_p^\pi(\brho)|r_{\max}}{1-\gamma}
\ge
2\Delta_b^{(0)},
\]
where we used $|\bar{\mathcal S}|\ge 1$, $\|1/\bar{\brho}\|_\infty\ge 1$, $\gamma|\bar{\mathcal A}|+1\ge 1$, and $(1-\gamma)^4\le 1$, we obtain
\[
\Delta_b^{(k)}
\le
\frac{128 (\gamma|\bar{\calA}|+1)|\bar{\calS}|r_{\max}|V^\pi_p(\brho)|}{(1-\gamma)^5 k}\left\|1/\bar{\brho}\right\|_{\infty}^2.
\]
Hence the choice of $M$ in Algorithm~\ref{alg:constrained_pg} guarantees that, for every $b\in\mathcal G_{h_\zeta}$,
\[
u_b\ge \varphi(b)\ge u_b-\zeta/3.
\tag{$\star$}
\]

\noindent\textit{Outer discretization.}
Let $b^\star\in\arg\min_{b\in\mathcal B_\sigma(\bo)}\varphi(b)$.
Define $\tilde b\in\mathcal G_{h_\zeta}$ by coordinatewise flooring:
\[
\tilde b_{sa}:=h_\zeta\left\lfloor b_{sa}^\star/h_\zeta\right\rfloor,\qquad (s,a)\in\mathcal S\times\mathcal A.
\]
Then $\tilde b\in\mathcal G_{h_\zeta}$, and Lemma~\ref{lem:product_box_mesh} together with the choice of $h_\zeta$ in Algorithm~\ref{alg:constrained_pg} gives
\[
\varphi(\tilde b)\le \varphi(b^\star)+\zeta/3=\uno(\sigma,\bo,p)+\zeta/3.
\]

\noindent\textit{Conclusion.}
Using $(\star)$ at $b=\tilde b$,
\[
u_\zeta
\le
u_{\tilde b}
\le
\varphi(\tilde b)+\zeta/3
\le
\uno(\sigma,\bo,p)+\frac{2\zeta}{3}.
\]
Equivalently,
\[
\uno(\sigma,\bo,p)\ge u_\zeta-\frac{2\zeta}{3}\ge u_\zeta-\zeta.
\]
On the other hand, from $(\star)$,
\[
u_\zeta=\min_{b\in\mathcal G_{h_\zeta}}u_b
\ge
\min_{b\in\mathcal G_{h_\zeta}}\varphi(b)
\ge
\min_{b\in\mathcal B_\sigma(\bo)}\varphi(b)
=
\uno(\sigma,\bo,p),
\]
where the last equality is Proposition~\ref{prop:product_box_reduction}.
This proves
\[
u_\zeta\ge \uno(\sigma,\bo,p)\ge u_\zeta-\zeta.
\]
\end{proof}
\subsection{Convergence rate of projected gradient descent}
\label{subsec:proof_covergence_rate_general}
\begin{definition}[gradient-mapping domination \citep{Xiao2022convergence}]
Consider an $L$-smooth function $f$ on a compact convex set $Q$.
We say that $f$ satisfies the gradient-mapping domination condition if there exists $\omega>0$ such that
\[
\|G_{L}(x)\|_2 \ge \sqrt{2\omega}\left(f(T_L(x))-f^\star\right),\qquad \forall x\in Q,
\]
where $f^\star=\min_{x\in Q}f(x)$ and
\[
T_L(x):=\proj_Q\left(x-\frac{1}{L}\nabla f(x)\right),
\qquad
G_L(x):=L(x-T_L(x)).
\]
\end{definition}

\begin{theorem}[\cite{Xiao2022convergence}]\label{thm:convergence_rate_general}
Consider the minimization of an $L$-smooth function $f$ over a compact convex set $Q$.
Suppose that $f$ satisfies the gradient-mapping domination condition with constant $\omega>0$.
Then the projected-gradient sequence
\[
x^{(k+1)}=\proj_Q\left(x^{(k)}-\frac{1}{L}\nabla f(x^{(k)})\right)
\]
satisfies, for all $k\ge 1$,
\[
f(x^{(k)})-f^\star
\le
\max\left\{\frac{4L}{\omega k},\ \left(\frac{\sqrt{2}}{2}\right)^k\left(f(x^{(0)})-f^\star\right)\right\}.
\]
\end{theorem}
\begin{proof}[Proof of Theorem~\ref{thm:convergence_rate_general}]
We obtain, for any $x \in Q$,
\begin{align*}
     f(x) - f(T_L(x)) & \ge \frac{1}{2L} \|G_L(x)\|_2^2
    \\
    & \ge \frac{\omega}{L}( f(T_L(x)) -f^\star)^2.
\end{align*}
where for the first inequality, we used Theorem~1 of \cite{nesterov2013gradient}, and for the second inequality, the gradient-mapping domination condition is used. We obtain, for each $s\ge 0$,
\begin{align}\label{eq:13}
    f(x^{(s)}) - f(x^{(s+1)}) \ge \frac{\omega}{L}(f(x^{(s+1)})- f^\star)^2.
\end{align}
Denote $\delta_s = f(x^{(s)}) - f^\star$; note that $\delta_s \ge 0$.
We obtain:
\begin{align*}
    & \qquad  f(x^{(s)}) -f(x^{(s+1)})
    =\delta_s -\delta_{s+1} \ge \frac{\omega}{L}\delta_{s+1}^2
    \\
    & \text{which is equivalent to }\frac{1}{\delta_{s+1}} - \frac{1}{\delta_s} \ge \frac{\omega}{L}\frac{\delta_{s+1}}{\delta_s}.
\end{align*}
Summing up the inequality from $s=0$ to $s = k-1$, we obtain:
\begin{align*}
    \frac{1}{\delta_k} - \frac{1}{\delta_0} \ge \frac{\omega}{L}\sum_{s=0}^{k-1}\frac{\delta_{s+1}}{\delta_{s}}.
\end{align*}

Using a constant $r \in (0,1)$, we define $n(k,r)$ as the number of times the ratio $\delta_{s+1}/\delta_s$ is at least $r$ among the first $k$ iterations. Let $c \in (0,1)$ be a constant. Suppose $n(k,r) \ge ck$; then $\delta_{s+1}/\delta_s\ge r$ for at least $\lceil ck\rceil$ values of $s$ in $\{0, \dots, k-1\}$. In this case,
\begin{align*}
   \frac{\omega}{L}rck\le  \frac{1}{\delta_k} - \frac{1}{\delta_0} \le \frac{1}{\delta_k}.
\end{align*}
Then, we derive that
\begin{align*}
    \delta_k \le \frac{L}{\omega rc k}.
\end{align*}
Otherwise, when $n(k,r) < ck$, it holds that $\delta_{s+1}/\delta_s<r $ at least $\lceil (1-c)k\rceil$ times.
From the descent property \eqref{eq:13}, we further obtain $\delta_{s+1} \le \delta_s$ for each $s \in \{0, \ldots, k-1\}$. Therefore, we get
\begin{align*}
    \delta_k = \frac{\delta_k}{\delta_{k-1}}\frac{\delta_{k-1}}{\delta_{k-2}}\cdots\frac{\delta_1}{\delta_0}\delta_0 < \delta_0 r^{(1-c)k}.
\end{align*}
Therefore, by taking $c=r=1/2$, we obtain,
\begin{align*}
    \delta_k \le \max\left\{\frac{4L}{\omega k},\left(\frac{1}{\sqrt{2}}\right)^{k}\delta_0\right\}.
\end{align*}
This concludes the proof.
\end{proof}

\subsection{Proof of Proposition~\ref{prop:product_box_reduction}}\label{subsec:product_box_reduction}

\begin{proof}[Proof of Proposition~\ref{prop:product_box_reduction}]
If $q\in\Pi_\sigma^p$, define $b(q)\in\mathbb R_+^{|\mathcal S||\mathcal A|}$ by
\[
b(q)_{sa}:=\kl_{sa}(p,q),\qquad (s,a)\in\mathcal S\times\mathcal A.
\]
Then $b(q)\in\mathcal B_\sigma(\bo)$ and $q\in Q(b(q))$.
Hence
\[
\uno(\sigma,\bo,p)
=
\min_{q\in\Pi_\sigma^p}V_p^\pi(\brho)V_q^\pi(\brho)
\ge
\min_{b\in\mathcal B_\sigma(\bo)}\ \min_{q\in Q(b)}V_p^\pi(\brho)V_q^\pi(\brho).
\]
Conversely, if $b\in\mathcal B_\sigma(\bo)$ and $q\in Q(b)$, then
\[
\sum_{s,a}\omega_{sa}\kl_{sa}(p,q)
\le
\sum_{s,a}\omega_{sa}b_{sa}
\le \sigma,
\]
so $q\in\Pi_\sigma^p$. Therefore,
\[
\min_{b\in\mathcal B_\sigma(\bo)}\ \min_{q\in Q(b)}V_p^\pi(\brho)V_q^\pi(\brho)
\ge
\uno(\sigma,\bo,p).
\]
This proves the equality.
\end{proof}

\section{Upper Bound on the Sample Complexity--Proof of Theorem~\ref{thm:sample}}\label{app:sample}
{\bf Convention.}
Throughout this section, for simplicity of presentation, we assume $\text{Ans}(p)=+$. %

\begin{proof}[Proof of Theorem~\ref{thm:sample}]
\underline{Show $\delta$-PC.}
Recall that Algorithm~\ref{alg:static_algorithm} stops in round $\tau$ only if $u_{\zeta_\tau}-\zeta_\tau \ge 0$.
Thanks to Theorem~\ref{thm:proj_pg_rate}, we have
\begin{equation}\label{eq:delta_PC-1}
\uno(\beta(\tau,\delta)/\tau,\hat{\bo}(\tau),\hat{p}_\tau)\ge u_{\zeta_\tau}-\zeta_{\tau}\ge 0.
\end{equation}
As Assumption~\ref{apt:general} implies that $\min_{q}V^\pi_q(\brho)V^\pi_{\hat{p}_\tau}(\brho)<0$, inequality (\ref{eq:delta_PC-1}) yields that 
\begin{equation}\label{eq:delta_PC_2}
\uno(\beta(\tau,\delta)/\tau,\hat{\bo}(\tau),\hat{p}_\tau)>\min_{q}V^\pi_q(\brho)V^\pi_{\hat{p}_\tau}(\brho).
\end{equation}
Observing that $\snc (\cdot,\hat{\bo}(\tau),\hat{p}_\tau)$ is a decreasing function, (\ref{eq:delta_PC-1}) also yields that 
\begin{equation}\label{eq:delta_PC_3}
    \snc(0,\hat{\bo}(\tau),\hat{p}_\tau)\ge \snc (\uno(\beta(\tau,\delta)/\tau,\hat{\bo}(\tau),\hat{p}_\tau),\hat{\bo}(\tau),\hat{p}_\tau)=\beta(\tau,\delta)/\tau,
\end{equation}
where the last equality follows from Proposition~\ref{prop:bijection in main} with $\sigma=\beta(\tau,\delta)/\tau\ge 0$ and condition (\ref{eq:delta_PC_2}). One can notice that (\ref{eq:delta_PC_3}) is equivalent to (\ref{eq:ideal stopping}). Hence, if $\ans (\hat{p}_\tau)\neq \ans(p)$ (in other words, if $p\in \alt(\hat{p}_\tau)$), then
\[
\sum_{s,a}N_{sa}(\tau)\kl_{sa}(\hat{p}_\tau,p)\ge \beta(\tau,\delta).
\]
By Proposition 1 in \cite{jonsson2020planning} (see (\ref{eq:beta guarantee})), we deduce that $\mathbb{P}_p[\ans (\hat{p}_\tau)\neq \ans(p)]\le \delta$.

\medskip
\noindent
\underline{Show the upper bound of sample complexity.}
For simplicity of presentation, we assume $\ans (p)=+$ in this proof; the case where $\ans(p)=-$ can be derived analogously. We first introduce the function
$$
F(\bo',p'):=\inf_{q\in \alt(p)}\sum_{s,a}\omega'_{sa}\kl_{sa}(p',q),\quad \forall \bo'\in \Sigma,\,p'\in\Ptest^+,
$$
and $\varepsilon\in (0,F(\bo,p)/2)$. As shown in Lemma~\ref{lem: F is cont} (Appendix~\ref{subsec: F is cont}), $F$ is a continuous function on $\Sigma\times \Ptest^+$; thus, there exists $\xi_1\in (0,1)$ such that 
\begin{equation}
    \left|F(\bo,p)-F(\bo',p')\right|<\varepsilon\quad \hbox{if }\max\{ \left\|\bo-\bo'\right\|_1,\left\|p'-p\right\|_1\}\le \xi_1.
\end{equation}
Moreover, an application of Theorem~\ref{thm:sensitivity_u} in Appendix~\ref{app:sensitivity} with $u=0$ and $p=p$ implies that there exist $\xi_2\in (0,1)$ and $c>0$ such that $\uno(\cdot,\hat{\bo}(t),\hat{p}_t)$ decays faster than a linear function $f(\sigma)=-c\sigma$ if $\left\|\hat{p}_t-p\right\|_1<\xi_2$ and $\left\|\hat{\bo}(t)-\bo\right\|_1<\xi_2$. We introduce $\xi=\min\{\xi_1,\xi_2\}$ and define the 'good event' 
\begin{equation}\label{eq:good event}
\mathcal{C}_T(\xi)=\bigcap_{\sqrt{T}\le t\le T}\left\{\max\left\{ \left\|\hat{\bo}(t)-\bo\right\|_{1},\left\|\hat{p}_t-p\right\|_1\right\}\le \xi\right\}.
\end{equation}
By Proposition~\ref{prop: good event} (proved later in this section), there exists $T_1(\xi)$ such that for $T\ge T_1(\xi)$, the event $\mathcal{C}_T(\xi)$ occurs with high probability. Moreover, as $\beta(T,\delta)+\zeta_TT/c=\log (1/\delta)+o(T)$ and $F(\bo,p)-\varepsilon>0$, one can find an integer $T_2(\xi)\in\mathbb{N}$ such that if $T\ge T_2(\xi)$,
\begin{equation}\label{eq:T2}
    \beta(T,\delta)+\frac{\zeta_TT}{c}\le \log(1/\delta)+(F(\bo,p)-\varepsilon)
    \xi T.
\end{equation}
Finally, we define 
\begin{equation}\label{eq:T3}
T_3(\xi,\varepsilon,\delta)=\frac{(F(\bo,p)-\varepsilon)^{-1}\log(1/\delta)}{1-\xi}.
\end{equation}
With these definitions, if $T\ge \max\{T_1(\xi),T_2(\xi),T_3(\xi,\delta)\}$, then conditional on $\mathcal{C}_T(\xi)$, we have 
\begin{align}\label{eq:sample-1}
    \nonumber \frac{\zeta_TT}{c}+\beta(T,\delta)&\le\log(1/\delta)+(F(\bo,p)-\varepsilon)\xi T\\
\nonumber     &\le (F(\bo,p)-\varepsilon)T\\
     &\le F(\hat{\bo}(T),\hat{p}_T)T,
\end{align}
where the first inequality is (\ref{eq:T2}); the second one follows from (\ref{eq:T3}); the last one is a consequence of (\ref{eq:good event}) and $T\ge T_1(\xi)$. Recall that $ F(\hat{\bo}(T),\hat{p}_T)$ is exactly $\snc(0,\hat{\bo}(T),\hat{p}_T) T$. Applying Theorem~\ref{thm:sensitivity_u} with $u=0,\,\hat{\bo} = \hat{\bo}(T),\,\hat{p} = \hat{p}_T,\,\sigma_1=\beta(T,\delta)/T,\,\sigma_2=\snc(0,\hat{\bo}(T),\hat{p}_T)$, (\ref{eq:sample-1}) implies that
\begin{align}   \label{eq:sample-2}
  \nonumber  \zeta_T&\le \uno(\beta(T,\delta)/T,\hat{\bo}(T),\hat{p}_T)-\uno(\snc(0,\hat{\bo}(T),\hat{p}_T),\hat{\bo}(T),\hat{p}_T)\\
    &=\uno(\beta(T,\delta)/T,\hat{\bo}(T),\hat{p}_T),
\end{align}
where the last equality follows from Proposition~\ref{prop:bijection in main}. By Theorem~\ref{thm:proj_pg_rate}, we have $u_{\zeta_T}\ge \uno(\beta(T,\delta)/T,\hat{\bo}(T),\hat{p}_T)$, hence (\ref{eq:sample-2}) implies $u_{\zeta_T}-\zeta_T\ge 0$. Therefore, $\tau\le T$. Namely, $\forall T\ge \max\{T_1(\xi),T_2(\xi),T_3(\xi,\varepsilon,\delta)\}$, $\mathcal{C}_T(\xi)\subseteq\{\tau\le T\}$. We can conclude that 
\begin{align}\label{eq:sample-3}
\nonumber    \mathbb{E}_p[\tau]&\le \max\{T_1(\xi),T_2(\xi),T_3(\xi,\varepsilon,\delta)\}+\sum_{T=\max\{T_1(\xi),T_2(\xi),T_3(\xi,\varepsilon,\delta)\}+1}^\infty\mathbb{P}_p[\tau>T]\\
    &\le \max\{T_1(\xi),T_2(\xi),T_3(\xi,\varepsilon,\delta)\}+\sum_{T=\lceil T_1(\xi)\rceil}^\infty\mathbb{P}_p[\mathcal{C}_T(\xi)^c].
\end{align}
From Proposition~\ref{prop: good event}, $\sum_{T=\lceil T_1(\xi)\rceil}^\infty\mathbb{P}_p[\mathcal{C}_T(\xi)^c]\le8\left|\mathcal{S}\right|^4\left|\mathcal{A}\right|^3/\xi^2\min_{s,a}w_{sa}$. As a consequence of (\ref{eq:sample-3}), 
\[
\limsup_{\delta\to 0}\frac{\mathbb{E}_p[\tau]}{\log (1/\delta)}\le\limsup_{\delta\to 0}\frac{T_3(\xi,\varepsilon,\delta)}{\log (1/\delta)} \le \frac{(F(\bo,p)-\varepsilon)^{-1}}{1-\xi}.
\]
As $\varepsilon,\xi$ can be taken arbitrarily small, the proof is  completed.

\end{proof}

\begin{proposition}\label{prop: good event}
Under Assumption~\ref{apt:sampling}, in Algorithm~\ref{alg:static_algorithm}, for any $\xi\in (0,1),\bo\in \Sigma$, there exists $T_1(\xi)>0$ such that 
\[\sum_{T=\lceil T_1(\xi)\rceil }^\infty \mathbb{P}_p\left[  \mathcal{C}_T(\xi)^c \right]<\frac{8\left|\mathcal{S}\right|^4\left|\mathcal{A}\right|^3}{\xi^2\min_{s,a}w_{sa}},
\]
where $\mathcal{C}_T(\xi)$ is introduced in (\ref{eq:good event}).
\end{proposition}
\begin{proof}
Due Assumption~\ref{apt:sampling}, $\mathcal{W}:=\{(s,a):\omega_{sa}>0\}=\mathcal{S}\times\mathcal{A}.$
   By Lemma~\ref{lem:tracking}, $\left\|\hat{\bo}(t)-\bo\right\|_1\le \sum_{s,a}\left|\mathcal{S}\right|\left|\mathcal{A}\right|/t\le \left|\mathcal{S}\right|^2\left|\mathcal{A}\right|^2/t$. We then derive that when $T\ge \left|\mathcal{S}\right|^4\left|\mathcal{A}\right|^4/\xi^2$ and $t\ge \sqrt{T}$, one can deduce that $\left\|\hat{\bo}(t)-\bo\right\|_1\le \xi$.\\
   As for the estimate on $p$, we apply Lemma~\ref{lem:tracking} again to have that for each $(s,a)$,  
\begin{equation}\label{eq:xi-1}
   N_{sa}(t)\ge t\min_{s,a}\omega_{sa}-\left|\mathcal{S}\right|\left|\mathcal{A}\right|\ge t\min_{s,a}\omega_{sa}/2 
\end{equation}
if $T\ge 4\left|\mathcal{S}\right|^2\left|\mathcal{A}\right|^2/\min_{s,a}w_{sa}^2$ and $t\ge \sqrt{T}$. Using the union bound yields that 
\begin{align*}
    \mathbb{P}_p\left[\left\|\hat{p}_t-p\right\|_1 \ge \xi\right]&\le \sum_{s,a} \mathbb{P}_p\left[\left\|\hat{p}_t(\cdot\mid s,a)-p(\cdot\mid s,a)\right\|_1 \ge \frac{\xi}{\left|\mathcal{S}\right|\left|\mathcal{A}\right|} \right]\\
    &\le  2\left|\mathcal{S}\right|\left|\mathcal{A}\right|\exp\left(-\frac{t\xi^2\min_{s,a}w_{sa}}{2\left|\mathcal{S}\right|^3\left|\mathcal{A}\right|^2}\right),
\end{align*}
where the last inequality follows from Lemma~\ref{lem:con on p} and (\ref{eq:xi-1}). By introducing $T_1(\xi)=\max\{\frac{\left|\mathcal{S}\right|^4\left|\mathcal{A}\right|^4}{\xi^2},\frac{4\left|\mathcal{S}\right|^2\left|\mathcal{A}\right|^2}{\min_{(s,a)\in\mathcal{W}}w^2_{sa}} \}$, union bound yields that
\begin{align*}
    \sum_{T=\lceil T_1(\xi)\rceil }^\infty\mathbb{P}_p[\mathcal{C}_T(\xi)^c]&\le\sum_{T=\lceil T_1(\xi)\rceil}^\infty\sum_{\sqrt{T}\le t\le T}2\left|\mathcal{S}\right|\left|\mathcal{A}\right|\exp\left(-\frac{t\xi^2\min_{s,a}w_{sa}}{2\left|\mathcal{S}\right|^3\left|\mathcal{A}\right|^2}\right)\\
    &\le \int_{1}^\infty\int_{\sqrt{T}}^T2\left|\mathcal{S}\right|\left|\mathcal{A}\right|\exp\left(-\frac{t\xi^2\min_{s,a}w_{sa}}{2\left|\mathcal{S}\right|^3\left|\mathcal{A}\right|^2}\right)dtdT.
\end{align*} 
The proof is completed by applying Lemma~\ref{lem:wang} with $A=\frac{t\xi^2\min_{s,a}w_{sa}}{2\left|\mathcal{S}\right|^3\left|\mathcal{A}\right|^2},\,\alpha=1/2,\,\beta=1$.
\end{proof}

\begin{lemma}[Proposition 1 in \cite{weissman2003inequalities}]\label{lem:con on p}
    Suppose one has samples the state-action pair $(s,a)$ for $n\ge 1$ times, then the empirical estimate on $p(\cdot\mid s,a)$, $\hat{p}_n(\cdot\mid s,a)$ satisfies that 
    \[
    \mathbb{P}\left[\left\| \hat{p}_n(\cdot\mid s,a)-p(\cdot\mid s,a)\right\|_1 \ge \varepsilon\right]\le 2e^{-\frac{n\varepsilon^2}{\left|\mathcal{S}\right|}},\quad \forall\varepsilon\in (0,1).
    \]
\end{lemma}
\medskip

\begin{lemma}\label{lem:tracking}
    Let $\boldsymbol{\omega}\in \Sigma$ and define $\mathcal{W}=\{(s,a):\omega_{sa}>0\}$. A sampling rule does 
    \begin{align*}
        A_t &\gets (s,a),& \hbox{if }(s,a)\in \mathcal{W} \hbox{ and }N_{sa}(t)=0,\\
        A_t&\gets  \argmin_{(s,a)}N_{sa}(t-1)/\omega_{sa}\text{ (tie-broken arbitrarily)}, &\hbox{otherwise.}
    \end{align*} 
    Then for each $(s,a)\in \mathcal{S}\times\mathcal{A}$ and $t\ge \left|\mathcal{W}\right|$, one has
    \begin{equation}\label{eq:tracking}
        t\omega_{sa}-\left|\mathcal{W}\right|    \le N_{sa}(t)\le t\omega_{sa}+1.
    \end{equation}
\end{lemma}
\begin{proof}
If $(s,a)\notin\mathcal{W}$, $N_{sa}(t)=0$ for all $t\in \mathbb{N}$, (\ref{eq:tracking}) holds directly.\\
For a fixed $(s,a)\in\mathcal{W}$, we prove the upper bound in (\ref{eq:tracking}) by induction. When $t=\left|\mathcal{W}\right|$, $N_{sa}(t)=1\le t\omega_{sa}+1$. Now suppose $ N_{sa}(t-1)\le (t-1)\omega_{sa}+1$, and consider two following cases, (i) $A_t\neq (s,a)$; (ii) $A_t=(s,a)$.\\
When (i) $A_t\neq (s,a)$, using the inductive hypothesis yields that
\[
N_{sa}(t)=N_{sa}(t-1)\le (t-1)\omega_{sa}+1\le t\omega_{sa}+1.
\]
As for (ii) $A_t=(s,a)$, one can observe that 
\[
\min_{s',a'}\frac{N_{s'a'}(t-1)}{\omega_{s'a'}}\le \frac{\min_{s'a'}N_{s'a'}(t-1)}{\max_{s'a'}\omega_{s'a'}}\le \frac{ \frac{t-1}{K}}{\frac{1}{K}}\le t-1\le t.
\]
Since $(s'a')$ is the minimizer,  
\[
\frac{N_{sa}(t)}{\omega_{sa}}=\frac{N_{sa}(t-1)}{\omega_{sa}}+\frac{1}{\omega_{sa}}\le 1+\frac{1}{\omega_{sa}}.
\]
Thus the upper bound in (\ref{eq:tracking}) is obtained by multiplying $t\omega_{sa}$ on the the both sides of the above inequality.\\

We now prove the lower bound in (\ref{eq:tracking}). Notice that
\[
N_{sa}(t)=t-\sum_{s'\neq s, a'\neq a} N_{s'a'}\ge t-\sum_{s'\neq s, a'\neq a}(t\omega_{s
a'}+1)\ge t\omega_{sa}+\left|\mathcal{W}\right|,
\]
where the second inequality is due to the upper bound in (\ref{eq:tracking}).

\end{proof}
\begin{lemma}[Lemma 5 in \cite{wang2021fast}]\label{lem:wang}
Let $\alpha, \beta \in (0, 1)$ and $A > 0$.
\[
\int_0^\infty \left( \int_{T^\alpha}^\infty \exp(-At^\beta) dt \right) dT = \frac{\Gamma\left(\frac{1}{\alpha \beta} + \frac{1}{\beta}\right)}{\beta A^{\frac{1}{\alpha \beta} + \frac{1}{\beta}}}.
\]
\end{lemma}

\section{Sensitivity Analysis on \texorpdfstring{$\uno$}{uNO}}\label{app:sensitivity}

\begin{theorem}\label{thm:sensitivity_u}
Suppose Assumptions~\ref{apt:general} and \ref{apt:sampling} hold. For any $p\in \mathcal{P},\,\,u\in \mathbb{R} $, there exist constants $c>0,\,\xi\in (0,\min_{sa}\omega_{sa}/2)$ such that if $\hat{p}\in \{q\in\mathcal{P}:\left\|p-q\right\|_1\le \xi\}$, $\hat{\bo}\in \{\bo'\in\Sigma:\left\|\bo'-\bo\right\|_1\le \xi\}$ and $0<\sigma_1<\sigma_2\le \bar{\sigma}$, where $\bar{\sigma}=\max_{\left\|\hat{\bo}-\bo\right\|_1\le \xi}\max_{\left\|\hat{p}-p\right\|_1\le \xi}\{\snc(u,\hat{\bo},\hat{p})\}$, then
\[
\uno(\sigma_2,\hat{\bo},\hat{p})-\uno(\sigma_1,\hat{\bo},\hat{p})\le -c(\sigma_2-\sigma_1).
\]
\end{theorem}
\begin{proof}
One can assume $\hat{p}$ is full-supported. Otherwise, due to the continuity of $\uno(\sigma, \hat{\bo}, \cdot)$ with respect to its third argument (the kernel), as shown in Lemma~\ref{lem:cont on uno} (Appendix~\ref{subsec:cont on uno}), for an arbitrary $\varepsilon>0$, one can always find a full-supported kernel $\tilde{p}$ sufficiently close to $\hat{p}$ such that
\[
\left|\uno(\sigma_2,\hat{\bo},\tilde{p})-\uno(\sigma_1,\hat{\bo},\tilde{p})-\uno(\sigma_2,\hat{\bo},\hat{p})-\uno(\sigma_1,\hat{\bo},\hat{p})\right|\le \varepsilon.
\]
Because $\uno(\sigma,\hat{\bo},\hat{p})$ is a decreasing function of $\sigma$, an application of Monotone difference lemma (\ref{lem:MDL}) implies that $\uno(\sigma,\hat{\bo}, \hat{p})$ as a function of $\sigma$ is differentiable almost everywhere. Let $\sigma\in[\sigma_1,\sigma_2]$ be a point at which $\uno (\sigma, \hat{\bo},\hat{p})$ is differentiable and $q_{\sigma}$ be the solution of (NO-$\sigma,\hat{\bo},\hat{p}$). Let $\eta_{\sigma}\in\mathbb{R}_+,\lambda_{s'sa}\in\mathbb{R}_+,\mu_{sa}\in \mathbb{R}$ be the Lagrange multipliers associated with the constraints $\sum_{s,a} \omega_{sa}\kl_{sa}(\hat{p},q)- \sigma\le 0$, $-q(s'\mid s,a)\le 0,\,\sum_{s'\in\mathcal{S}}q(s'\mid s,a)-1=0$ respectively. 
An application of Envelope Theorem (Theorem~\ref{thm:envelope_for_constrained}) yields that
\begin{align*}
 \frac{\partial}{\partial \sigma}\uno(\sigma,\hat{\bo},\hat{p})&=\frac{\partial}{\partial \sigma}\left( V^\pi_{\hat{p}}(\brho)V^\pi_{q_\sigma}(\brho)\right)+\eta_{\sigma} \frac{\partial}{\partial \sigma}\left(\sum_{s,a}\hat{\omega}_{sa}\kl_{sa}(\hat{p},q_{\sigma})-\sigma\right)\\
 &+\frac{\partial}{\partial\sigma}\sum_{s',s,a}\lambda_{s'sa}(-q_\sigma(s'\mid s,a))+\frac{\partial}{\partial\sigma} \sum_{s,a}\mu_{sa} \left(\sum_{s'\in\mathcal{S}} q_\sigma(s'\mid s,a)-1\right)=-\eta_{\sigma}  
\end{align*}
By Lemma~\ref{lem:eta}, we know 
\begin{equation}\label{eq:trade-off-1}
\eta_{\sigma}=\frac{\gamma V^\pi_{\hat{p}}(\brho)d_{q_\sigma,\brho}(s,a)\left(V^\pi_{q_\sigma}(s^M_{\sigma})-V^\pi_{q_{\sigma}}(s^m_{\sigma})\right)}{\hat{\omega}_{sa}(1-\gamma)\left(\frac{\hat{p}(s^M_{\sigma}\mid s,a)}{q_{\sigma}(s^M_{\sigma}\mid s,a)}-\frac{\hat{p}(s^m_{\sigma}\mid s,a)}{q_\sigma(s^m_{\sigma}\mid s,a)}\right)},\quad \forall s\in\mathcal{S},a\in\mathcal{A}.
\end{equation}
    where $s^M_{\sigma}\in \argmax_sV^\pi_{\hat{p}}(\brho)V^\pi_{q_\sigma}(s)$ and $s^m_{\sigma}\in \argmin_sV^\pi_{\hat{p}}(\brho)V^\pi_{q_\sigma}(s)$. By invoking the fundamental theorem of calculus, we have
\begin{align*}
\uno(\sigma_2,\hat{\bo},\hat{p})-    \uno(\sigma_1,\hat{\bo},\hat{p})&=\int_{\sigma_1}^{\sigma_2}-\eta_{\sigma} d\sigma.%
\end{align*}
It suffices to show $\eta_\sigma> c$ for some $c>0$. Notice that if $r$ and $\brho$ satisfy Assumption~\ref{apt:general}, so do $rV^{\pi}_{p}(\brho)$ and $\brho$.  Lemma~\ref{lem:from apt} then implies that  $\min_{q\in\mathcal{P}}\max_{s,s'}V^\pi_p(\brho)V^\pi_q(s)-V^\pi_p(\brho)V^\pi_q(s')>0$. As $V^\pi_{\cdot}(\brho)$ is a continuous function, there exists $c_1>0,\,\xi\in (0,\min_{sa}\omega_{sa}/2)$ such that 
\begin{equation}
\min_{q\in\mathcal{P}}\max_{s,s'}V^\pi_{\hat{p}}(\brho)V^\pi_q(s)-V^\pi_{\hat{p}}(\brho)V^\pi_q(s') \ge c_1,\quad \forall \left\|\hat{p}-p\right\|_1<\xi.
\end{equation}
Further observe that 
\begin{align*}
   \frac{\hat{p}(s^M\mid s,a)}{q(s^m\mid s,a)}-\frac{\hat{p}(s^M\mid s,a)}{q(s^m\mid s,a)}\le 2\max_{s',s,a}\frac{\hat{p}(s'\mid s,a)}{q(s'\mid s,a)}\le 2\max_{q:\sum_{sa}w_{sa}\kl_{sa}(\hat{p},q)\le \bar{\sigma}}\max_{s',s,a}\frac{\hat{p}(s'\mid s,a)}{q(s'\mid s,a)} ,
\end{align*}
which is upper bounded by some $c_2>0$ for any $\left\|\hat{p}-p\right\|_1\le \xi$. Hence the proof is completed by setting $c=\frac{\gamma c_1}{(1-\gamma)c_2}\min_{sa}\rho(s)\pi(a\mid s)$, where $\min_{sa}\rho(s)\pi(a\mid s)>0$ thanks to Assumptions~\ref{apt:general} and \ref{apt:sampling}.
\end{proof}
\subsection{Technical Lemmas}\label{subsec:tech_lem}

\begin{lemma}[Monotone difference lemma, see e.g. Theorem 1.6.25 in \cite{tao2011introduction}]\label{lem:MDL}
Any function $F:\mathbb{R}\mapsto\mathbb{R}$ which is monotone is differentiable almost everywhere.
\end{lemma}

\begin{theorem}[Kuhn-Tucker Theorem,  Theorem A.30 in \cite{acemoglu2008introduction}]\label{thm:KT}
    Consider the constrained minimization problem
    \begin{align*}
    &\inf_{x \in \mathbb{R}^K} f(x) \\
    \text{s.t. } 
    &g(x) \leq 0 \quad \text{and} \quad h(x) = 0,
\end{align*}
where  $f : x \in X \rightarrow \mathbb{R}, \, g : x \in X \rightarrow \mathbb{R}^N,\, h : x \in X \rightarrow \mathbb{R}^M \, (\text{for some } K, N, M \in \mathbb{N})$ and $X\subset \mathbb{R}^K$ is a vector space. Let $ x^* \in X$ be a solution to this minimization problem, and suppose that  $N_1 \leq N$ of the inequality constraints are active, in the sense that they hold as equality at  $x^*$. Define $\tilde{h} : X \rightarrow \mathbb{R}^{M + N_1}$ to be the mapping of these $N_1$ active constraints stacked with $h(x)$ (so that $\tilde{h}(x^*) = 0)$. Suppose that the following constraint qualification condition is satisfied: the Jacobian matrix $D_x(\tilde{h}(x^*))$ has rank $N_1 + M$. Then the following Kuhn-Tucker condition is satisfied: there exist Lagrange multipliers $\boldsymbol{\lambda}^* \in \mathbb{R}^{N_1}$ and $\boldsymbol{\mu}^* \in \mathbb{R}^M $ such that
\begin{equation*}
    D_x f(x^*) + \boldsymbol{\lambda}^* \cdot D_x g(x^*) + \mu^* \cdot D_x h(x^*) = 0, %
\end{equation*}
and the complementary slackness condition
\begin{equation*}
    \boldsymbol{\lambda}^* \cdot g(x^*) = 0
\end{equation*} 
holds
\end{theorem}

\begin{theorem}[Envelope Theorem for constrained optimization problem, Theorem A.31 in \cite{acemoglu2008introduction}]\label{thm:envelope_for_constrained}
    Consider the constrained minimization problem
    \begin{align*}
        v(p) & = \min_{x \in X} f(x, p)
        \\
         \text{s.t.} \quad &g(x, p) \le 0,\,\hbox{and } h(x,p)=0,
    \end{align*}
    where $X\subset \mathbb{R}^K$ is a vector space, $p\in\mathbb{R}$; and $f:X\times\mathbb{R}\to\mathbb{R},g:X\times\mathbb{R}^N\to \mathbb{R}$, and $h:X\times\mathbb{R}^M\to\mathbb{R}$ are differentiable ($K,N,M\in\mathbb{N}$). 
    Let $x^*(p)\in \textnormal{Int}(X)$ be a solution to the problem. Denote the Lagrangian multipliers associated with the inequality and equality by $\boldsymbol{\lambda}^*\in \mathbb{R}_+^N$ and $\boldsymbol{\mu}^*\in \mathbb{R}^M$.
    Suppose also $v(p)$ is differentiable at $\bar{p}$. Then we have
    \begin{align*}
        \frac{d v(\bar{p})}{d p} = \frac{\partial f(x^*(\bar{p}), \bar{p})}{\partial p} + \boldsymbol{\lambda}^*D_pg(x^*(\bar{p}),\bar{p}) +\boldsymbol{\mu}^* D_ph(x^*(\bar{p}),\bar{p}).
    \end{align*}
\end{theorem}

\subsection{The Value of the Lagrangian Multiplier}
\begin{lemma}\label{lem:eta}
Suppose Assumption~\ref{apt:general} and \ref{apt:sampling}
hold. Let $\sigma>0$ and $p\in\mathcal{P}$ is full-supported. Denote $q_\sigma$ as the solution to (NO-$\sigma,\bo,p$). Then the Lagrange multiplier associated with the inequality $\sum_{s,a}\omega_{sa}\kl_{sa}(p,q)\le \sigma$ is
    \begin{equation}\label{eq:eta_sigma}
        \eta_\sigma=\frac{\gamma V^\pi_p(\brho)d_{q_\sigma,\brho}(s,a)\left(V^\pi_{q_\sigma}(s^M_\sigma)-V^\pi_{q_\sigma}(s^m_\sigma)\right)}{\omega_{sa}(1-\gamma)\left(\frac{p(s^M_\sigma\mid s,a)}{q_\sigma(s^M_\sigma\mid s,a)}-\frac{p(s^m_\sigma\mid s,a)}{q_\sigma(s^m_\sigma\mid s,a)}\right)}\quad \forall s\in\mathcal{S},a\in\mathcal{A}.
    \end{equation}
    where $s^M_\sigma \in \argmax_{s\in\mathcal{S}} V^\pi_{p}(\brho)\,V^\pi_{q_\sigma}(s)$ and $s^m_\sigma \in \argmin_{s\in\mathcal{S}} V^\pi_{p}(\brho)\,V^\pi_{q_\sigma}(s)$. %
\end{lemma}
\begin{proof}
Let $\sigma>0$. The Lagrangian function of the optimization problem (NO-$\sigma,\bo,p$) is 
\begin{align*}
L(q,\eta, \boldsymbol{\lambda},\boldsymbol{\mu})&= V^\pi_p(\brho)V^\pi_q(\brho)+\eta\left(\sum_{s,a}\omega_{sa}\kl_{sa}(p,q)-\sigma\right)\\
&+\sum_{s',s,a}\lambda_{s'sa}\left(-q(s'\vert s,a)\right)+\sum_{s,a}\mu_{sa}\left(\sum_{s'\in\mathcal{S}}q(s'\vert s,a)-1\right),    
\end{align*}
where $\eta \ge 0,\,\boldsymbol{\lambda}\in \mathbb{R}^{\left|\mathcal{S}\right|^2\left|\mathcal{A}\right|}_+$ and $\boldsymbol{\mu}\in \mathbb{R}^{\left|\mathcal{S}\right|\left|\mathcal{A}\right|}$. As $p$ is full-supported, $q_\sigma$ is full-supported as well (otherwise, it violates the constraint that $\sum_{s,a}\omega_{sa}\kl_{sa}(p,q_\sigma)\le \sigma$). That is, $q_\sigma(s'\vert s,a)>0,\,\forall s,s',a$. Further using Corollary~\ref{cor:min in Psigma} in Appendix~\ref{subsec:prop of uno}, we conclude that $\sum_{s,a}\omega_{sa}\kl_{sa}(p,q_\sigma)= \sigma$. In other words, the upper bound of the weighted KL-divergence is the only active inequality. We now prove that $\{D_{q} \sum_{s'\in\mathcal{S}}q_\sigma(s'\vert s,a)-1\}_{s\in\mathcal{S},a\in\mathcal{A}}\cup \{D_q \left(\sum_{s,a}\omega_{sa}\kl_{sa}(p,q_\sigma)- \sigma\right)\} $ is linear independent. Suppose on the contrary, $D_q \left(\sum_{s,a}\omega_{sa}\kl_{sa}(p,q_\sigma)- \sigma\right)$ is spanned by $\{D_{q} \sum_{s'\in\mathcal{S}}q_\sigma(s'\vert s,a)-1\}_{s\in\mathcal{S},a\in\mathcal{A}}$. As
\begin{align*}
    \frac{\partial}{\partial q(s'\vert s,a)}\sum_{s,a}\omega_{sa}\kl_{sa}(p,q_\sigma)- \sigma &=-\frac{\omega_{sa}p(s'\vert s,a)}{q_\sigma(s'\vert s,a)}\quad \forall s',s,a,\\
    \hbox{and }\frac{\partial}{\partial q(s'\vert s,a)}\sum_{s'\in\mathcal{S}}q_\sigma(s'\vert s,a)-1&=1,\quad \forall s',s,a.
\end{align*}
we deduce that $q_\sigma=p$ which contradicts that 
 $\sum_{s,a}\omega_{sa}\kl_{sa}(p,q_\sigma)= \sigma$. 
 Hence, we can apply Kuhn-Tucker Theorem (Theorem~\ref{thm:KT}) and obtain that there exists $\eta_\sigma\ge 0,\boldsymbol{\lambda}\in \mathbb{R}_+^{\left|\mathcal{S}\right|^2\left|\mathcal{A}\right|}$ and $\boldsymbol{\mu}\in \mathbb{R}^{\left|\mathcal{S}\right|\left|\mathcal{A}\right|}$ such that
\begin{align*}
    &\frac{\partial }{\partial q(s'\vert s,a)}V^\pi_p(\brho)V^\pi_{q_\sigma}(\brho) -\frac{\eta\omega_{sa}p(s'\vert s,a)}{q_\sigma(s'\vert s,a)}+\mu_{sa}-\lambda_{s'sa}&=0,\quad\forall s',s,a,\tag{Stationarity}\\
    \hbox{and }&\eta\left(\sum_{s,a}\omega_{sa}\kl_{sa}(p,q_\sigma)-\sigma\right)=0,\,\lambda_{s'sa}(-q_\sigma(s'\vert s,a))&=0,\quad\forall s',s,a,\tag{Complementary slackness}
\end{align*}
As $q_\sigma$ is full-supported, we derive $\lambda_{s'sa}=0,\forall s',s,a,$ from Complementary slackness. From (\ref{eq:PG2}) in Lemma~\ref{lem:reversed_PG}, Stationarity can be rewritten as:
\begin{equation}\label{eq:Lagrangian-1}
    \frac{V^\pi_p(\brho)}{1-\gamma}d_{q_\sigma,\brho}(s,a)\left(r(s,a)+\gamma V^\pi_{q_\sigma}(s')\right)-\frac{\eta\omega_{sa}p(s'\mid s,a)}{q_\sigma(s'\mid s,a)}=-\mu_{sa},\,\forall s',s,a.
\end{equation}
By taking difference of the equations (\ref{eq:Lagrangian-1}) with $s'=s^M_\sigma$ and $s'=s^m_\sigma$, we obtain
\[
\frac{\gamma V^\pi_p(\brho)d_{q_\sigma,\brho}(s,a)}{1-\gamma}\left(V^\pi_{q_\sigma} (s^M_\sigma)-V^\pi_{q_\sigma} (s^m_\sigma)\right)-\eta\omega_{sa}\left(\frac{p(s^M_\sigma\mid s,a)}{q_\sigma(s^M_\sigma\mid s,a)}-\frac{p(s^m_\sigma\mid s,a)}{q_\sigma(s^m_\sigma\mid s,a)}\right)=0.
\]
(\ref{eq:eta_sigma}) follows from a simple rearrangement on the above equation.
\end{proof}

\subsection{Properties for the Stationary Points}\label{subsec:prop of uno}
For the clarity of presentation, here we fix some $p\in\Ptest$, $\bo\in \Sigma$ and introduce the constrained set, 
\[
\mathcal{P}_\sigma:=\left\{ q\in\mathcal{P}:\sum_{sa}\omega_{sa}\kl_{sa}(p,q)\le \sigma\right\}.
\]
The goal of this subsection is to prove Corollary~\ref{cor:min in Psigma}, where we show the minimizer $q_\sigma\in\arg\min_{q\in \mathcal{P}_\sigma}V^\pi_p(\brho)V^\pi_{q}(\brho)$ satisfies that $\sum_{sa}\omega_{sa}\kl_{sa}(p,q_\sigma)=\sigma$. For this purpose, we firstly consider the stationary points in Lemma~\ref{lem:stationary_bd}.

\begin{lemma}\label{lem:stationary_bd}
Consider the optimization problem, $\min_{q\in \mathcal{P}_\sigma}V^\pi_p(\brho)V^\pi_{q}(\brho)$ under Assumption~\ref{apt:general}, \ref{apt:sampling}. All the stationary points will be on the boundary of $\mathcal{P}_\sigma$\footnote{The interior (boundary resp.) is referred to relatively interior, i.e. the topological interior (boundary) relative to the affine hull of the simplex. Interested readers are referred to \cite{zalinescu2002convex}. }. %
    
\end{lemma}
\begin{proof}
    Suppose on the contrary, there is a stationary point, say $q_o$, at the interior of $\mathcal{P}_\sigma$.
     As $q_o$ is a stationary point, one has $\langle q-q_o,\nabla V^\pi_{q_o}(\brho)V^\pi_{p}(\brho) \rangle\ge 0$ for all $q\in\mathcal{P}_\sigma.$ By invoking Lemma~\ref{lem:stationary+interior}, we derive that $\forall (s',a')\in\mathcal{S}\times \mathcal{A}$, $\exists \alpha_{s'a'} \in\RR$ such that for each $s''\in \mathcal{S}$,
     \begin{align}\label{eq:alpha_sa}
       \alpha_{s'a'}&=\frac{\partial V^\pi_{q_o}(\brho)V^\pi_{p}(\brho)}{\partial q(s''\vert s',a')}=\frac{V^\pi_{p}(\brho)}{1-\gamma}\sum_{s,a}\rho(s)\pi(a\vert s)d_{q_o,s,a}^\pi(s',a')(r(s,a)+\gamma V_{q_o}^\pi(s'')),
     \end{align}
     where the last equation stems directly from Lemma~\ref{lem:reversed_PG}.
     Let $\alpha :=\sum_{s',a'}\alpha_{s'a'}/V^\pi_{p}(\brho)$ and sum (\ref{eq:alpha_sa}) over all $s',a'\in\mathcal{S}\times \mathcal{A}$, one has 
     $$
     r^\pi(\brho)+\gamma V_{q_o}^\pi(s'')=\alpha,\,\forall s''\in \mathcal{S},
     $$
     which yields that $\forall s\in\mathcal{S},\,V_{q_o}^\pi(s)=\alpha':=(\alpha-r^\pi_{\brho})/\gamma$. However, it contradicts Lemma~\ref{lem:from apt}, hence the stationary points are on the boundary of $\mathcal{P}_\sigma$. 
\end{proof}

\begin{corollary}\label{cor:min in Psigma}
    Consider the minimizer $q_\sigma\in \argmin_{q\in \mathcal{P}_\sigma}V^\pi_p(\brho)V^\pi_{q}(\brho)$, one has 
    \begin{equation}\label{eq:touch bd}
        \sum_{sa}\omega_{sa}\kl_{sa}(p,q_\sigma)=\sigma
    \end{equation} if $p$ is full-supported and Assumption~\ref{apt:general}, \ref{apt:sampling} hold.
\end{corollary}
\begin{proof}
    Observe that $p(s'\mid s,a)>0$ for all $s',s,a$, hence $q_\sigma(s'\mid s,a)>0$ for all $s',s,a$. Otherwise, it violates $\sum_{sa}\omega_{sa}\kl_{sa}(p,q_\sigma)\le \sigma$. Moreover, as $q_\sigma$ is the stationary point, together with the conclusion from Lemma~\ref{lem:stationary_bd}, $q_\sigma$ is on the boundary, we deduce (\ref{eq:touch bd}).
\end{proof}

\begin{lemma}\label{lem:stationary+interior}
    Consider $v\in \RR^{\left|\mathcal{S}\right|^2\left| \mathcal{A}\right|}$ and an interior point of $\mathcal{P}_\sigma$, denoted by $q_o$. If $\langle q-q_o,v\rangle:=\sum_{s,a}\sum_{s'}(q(s'\vert s,a)-q_o(s'\vert s,a))v_{s',s,a}\ge 0$ for all $q\in\mathcal{P}_\sigma$, then for each $(s,a) \in\mathcal{S}\times \mathcal{A}$, $\exists\alpha_{sa}$ such that  $v_{s',s,a}=\alpha_{sa},\,\forall s',s\in\mathcal{S} ,a\in\mathcal{A}$. In other words, for each $(s,a)\in \mathcal{S}\times\mathcal{A}$ $v_{\cdot,s,a}$ is parallel to a $\left|\mathcal{S}\right|$-dimensional vector whose components are all $1$'s. 
\end{lemma}
\begin{proof}
    Let $(s,a)\in \mathcal{S}\times \mathcal{A}$ and $u\in \RR^{\left|\mathcal{S}\right|^2\left| \mathcal{A}\right|}$ such that $u_{s'',s',a'}=0$ if $(s',a')\neq (s,a)$ for each $s''\in\mathcal{S}$, and $\sum_{s'\in\mathcal{S}}u_{s',s,a}=0$. As $p_o$ is an interior point in $\mathcal{P}_\sigma$, one can find a small constant $c>0$ such that $q^+(s'\vert s,a):=q_o(s'\vert s,a)+cu_{s',s,a}$ and $q^-:=q_o(s'\vert s,a)-cu_{s',s,a},\,\forall s,s'\in\mathcal{S},\,a\in\mathcal{A}$ be two policies in $\mathcal{P}_\sigma$. From the assumption on $v$, we have $c\langle u ,v\rangle  = \langle q^+-q_o,v \rangle\ge 0$ and $-c\langle u ,v\rangle  = \langle q^--q_o,v \rangle\ge 0$, which implies that $\langle u,v\rangle =0$. As the only constraint of $u_{\cdot,s,a}$ is $\sum_{s'\in\mathcal{S}}u_{s',s,a}=0$, we deduce that $v_{\cdot,s,a}$ is parallel to a $\left|\mathcal{S}\right|$-dimensional vector whose components are all $1$'s
\end{proof}

\section{Maximal Theorem and its Applications}\label{subsec:Perturbation}
Suppose $\mathbb{X}$ and $\mathbb{Y}$ are Hausdorff topological spaces. Let $\psi:\mathbb{X}\times \mathbb{Y}\to \RR$ be a function and $\Phi:\mathbb{X}\setmap \mathbb{S}(\mathbb{Y})$ be a set-valued function, where $\mathbb{S}(\mathbb{Y})$ is the set of non-empty subsets of $\mathbb{Y}$. Furthermore, we introduce $\mathbb{K}(\mathbb{X})=\{F\in \mathbb{S}(\mathbb{X}):F \hbox{ is compact}\}$. We are interested in a minimization problem of the form:
\begin{align*}
v(x) &= \inf_{y\in \Phi(x)} \psi(x, y) \: ,\\
\Phi^*(x) &= \{y\in\Phi(x)\: : \: \psi(x, y) = v(x)\} \: .
\end{align*}
For $U\subset \mathbb{X}$, let the graph of $\Phi$ restricted to $U$ be $Gr_U(\Phi) = \{(x,y)\in U\times \mathbb{Y} \: : \: y\in\Phi(x)\}$ .
\begin{theorem}[Maximal theorem \citep{berge1877topological}]\label{thm:berge}
	Let $\mathbb{X}$ and $\mathbb{Y}$ be Hausdorff topological spaces. Assume that
	\begin{itemize}
		\item $\Phi: \mathbb{X} \rightrightarrows \mathbb{K}(\mathbb{X})$ is continuous (i.e. both lower and upper hemicontinous),
		\item $\psi: \mathbb{X} \times \mathbb{Y} \to \RR$ is continuous.
	\end{itemize}
	Then the function $v:\mathbb{X} \to \RR$ is continuous and the solution multifunction $\Phi^*:\mathbb{X}\to\mathbb{S}(\mathbb{Y})$ is upper hemicontinuous and compact valued.
\end{theorem}
\subsection{Proof of Lemma~\ref{lem: F is cont}}\label{subsec: F is cont}
Here we divide $\Ptest$ into two disjoint sets, $\Ptest=\Ptest^+\cup \Ptest^-$, where $\Ptest^+:=\{p\in\Ptest:\text{Ans}(p)=+\},\,\Ptest^-:=\{p\in\Ptest:\text{Ans}(p)=-\}.$
\begin{lemma}\label{lem: F is cont}
    A function $F:\Sigma\times \Ptest^+\to \mathbb{R}$ defined as  $F(\boldsymbol{\omega},q):=\inf_{q\in \alt (p)}\sum_{s,a}\omega_{sa}\kl_{sa}(p,q)$ is a continuous function.
\end{lemma}
\begin{proof}
The proof is established by invoking Theorem~\ref{thm:berge} with the following substitution:
\begin{align*}
&\mathbb{X}=\Sigma\times\Ptest^+,  &\psi(\bo,p,q)= \sum_{s,a}\omega_{sa}\kl_{sa}(p,q), & \\
&\mathbb{Y}=\mathcal{P}, &\Phi(\bo,p)=\cl(\alt (p))&.
\end{align*}
         Observe that objective function, $\sum_{s,a}\omega_{sa}\kl_{sa}(p,q)$, is a continuous function on $\Sigma\times \Ptest^+\times \mathcal{P}$, and the corresponding $\Phi (\boldsymbol{\omega},p)=\cl(\alt (p))$ is always a constant. %
\end{proof}
\subsection{Proof of Lemma~\ref{lem:cont on uno}}\label{subsec:cont on uno}
\begin{lemma}\label{lem:cont on uno}
Let $\bo\in \Sigma$ such that $\omega_{sa}>0,\,\forall (s,a)\in \mathcal{S}\times\mathcal{A}$. $\uno(\cdot, \bo,\cdot)$ is a continuous function on $\mathbb{R}_+\times \mathcal{P}$.
\end{lemma}
\begin{proof}
 
The proof is established by invoking Theorem~\ref{thm:berge} with the following substitution:
\begin{align*}
&\mathbb{X}=\mathbb{R}_+\times \mathcal{P},  &\psi(\sigma,p,q)= V^\pi_p(\brho)V^\pi_q(\brho), & \\
&\mathbb{Y}=\mathcal{P}, &\Phi(\sigma,p)=\{q\in \mathcal{P}:\sum_{sa}\omega_{sa}\kl_{sa}(p,q)\le \sigma\}&.
\end{align*}
As the objective function, $\psi$, is a continuous function on $\mathbb{R}_+\times \mathcal{P}$, it suffices to show the corresponding $\Phi(\sigma, p)=\{q\in\mathcal{P}: \sum_{s,a}\omega_{sa}\kl_{sa}(p,q)\le \sigma\}$ is hemi-continuous.\\

\underline{Show the upper hemi continuous of $\Phi$.} Let $\{p_n\}_{n=1}^\infty\subset \mathcal{P}$ and $\{\sigma_n\}_{n=1}^\infty\subset \mathbb{R}_+$ such that $p_n\xrightarrow{n\to\infty}p$ and $\sigma_n\xrightarrow{n\to\infty}\sigma$. And consider a sequence $\{q_n\}_{n=1}^\infty\subset \mathcal{P}$ such that $q_n\in \Phi(\sigma_n,p_n)$, which is equivalent to $\sum_{s,a}\omega_{sa}\kl_{sa}(p_n,q_n)-\sigma_n\le 0$, and $q_n\xrightarrow{n\to\infty} q$. By continuity of KL-divergence, one has
\[
\sum_{s,a}\omega_{sa}\kl_{sa}(p,q) -\sigma=\lim_{n\to\infty}\sum_{s,a}\omega_{sa}\kl_{sa}(p_n,q_n)-\sigma_n\le 0,
\]\\
which yields that $q\in \Phi(\sigma,p)$ and hence $\Phi$ is upper hemi-continuous.\\
\underline{Show the lower hemi continuous of $\Phi$.} Let $\{p_n\}_{n=1}^\infty\subset \mathcal{P}$ and $\{\sigma_n\}_{n=1}^\infty\subset \mathbb{R}_+$ be the sequences such that $p_n\xrightarrow{n\to \infty}p$ and $\sigma_n\xrightarrow{n\to \infty}\sigma$ for some $p\in\mathcal{P}$ and $\sigma\in \mathbb{R}$. Consider $q\in\Phi(\sigma, p)$, we now aim to show that $\exists \{q_n\}_{n=1}^\infty$ such that $q_n\in \Phi(\sigma_n,p_n)$ and $q_n\xrightarrow{n\to \infty} q$. The proof is separated into two cases (i) $\sum_{s,a}\omega_{sa}\kl_{sa}(p,q) < \sigma$ and (ii) $\sum_{s,a}\omega_{sa}\kl_{sa}(p,q) = \sigma$.\\
{\bf Case (i)} As $(\sigma_n, p_n)\xrightarrow{n\to\infty}(\sigma,p)$ and the continuity of KL-divergence, $\exists N>0$ such that $\sum_{s,a}\omega_{sa}\kl_{sa}(p_n,q) - \sigma_n<0$ for all $n\ge N$. Choosing $q_n=q$ yields the conclusion.\\ {\bf Case (ii)} For each $n\in \NN$, we define $\alpha_n=\max\{\alpha\in[0,1]: (1-\alpha)p_n+\alpha q\in\Phi(\sigma_n,p_n)\}$ and $q_n= (1-\alpha_n)p_n+\alpha_n q$, which directly implies $q_n\in\Phi(\sigma_n,p_n)$. From the definition of $\alpha_n$, when $\sum_{s,a}\omega_{sa}\kl_{sa}(p_n, q)=0$, $\alpha_n=1$. When $\sum_{s,a}\omega_{sa}\kl_{sa}(p_n, q)\neq 0$
Due to the joint convexity, one has 
$$ 
\sum_{s,a}\omega_{sa}\kl_{sa}(p_n,(1-\alpha)p_n+\alpha q)\le \alpha \sum_{s,a}\omega_{sa}\kl_{sa}(p_n, q).
$$    
Hence $\alpha \ge \frac{\sigma_n}{\sum_{s,a}\omega_{sa}\kl_{sa}(p_n, q)}$. In summary,
\[
\alpha_n\left\{\begin{array}{cc}
     \ge \frac{\sigma_n}{\sum_{s,a}\omega_{sa}\kl_{sa}(p_n, q)},&\hbox{ if } \sum_{s,a}\omega_{sa}\kl_{sa}(p_n, q)\neq 0, \\
     =1,& \hbox{otherwise.}
\end{array}\right.
\]
Using $(\sigma_n,p_n)\xrightarrow{n\to\infty}(\sigma,p)$ and the continuity of KL-divergence again, we have $\alpha_n\to 1$ and $q_n\to q$ as $n\to\infty$.

\end{proof}

\subsection{Proof of Lemma~\ref{lem:uno cont}}\label{subsec:uno cont}
\begin{lemma}\label{lem:uno cont}
Under Assumption~\ref{ass:hg_new} and the notion in Appendix~\ref{app:dual},
    $\uno(\cdot)$ is continuous in $[0,\infty)$.
\end{lemma}
\begin{proof}
We prove it by applying Theorem~\ref{thm:berge} with the following substitution:
\begin{align*}
&\mathbb{X}=[0,\infty),  &\psi(\sigma,x)=g(x), & \\
&\mathbb{Y}=\mathcal{X}, &\Phi(\sigma)=\{x\in \mathcal{X}:h(x)\le \sigma\}.&
\end{align*} As $\mathcal{X}$ is compact and $\psi$ is continuous according to Assumption~\ref{ass:hg_new}-(c), $\Phi(\sigma)$ is always a compact set. Additionally, $h(\underline{x})\le 0$ from Assumption~\ref{ass:hg_new}-(a), $\Phi(\sigma)\neq\emptyset$ for all $\sigma\ge 0$. It remains to show $\Phi(\cdot)$ is a continuous corresponding.

\noindent
\underline{Upper hemicontinuity.} Let $\{\sigma_n\}_{n=1}^\infty\subset \mathbb{X}$ and $\{x_n\}_{n=1}^\infty\subset \mathcal{X}$ be the sequences such that $x_n\in \Phi(\sigma_n)$, or equivalently $h(x_n)\le\sigma_n,\,\forall n\in \NN$,  $\lim_{n\to\infty }\sigma_n=\sigma^\star$, and $\lim_{n\to\infty}x_n=x^\star$. Since the continuity of $h$ is assumed in Assumption~\ref{ass:hg_new}-(d), we derive 
$$
h(x^\star)=\limsup_{n\to\infty} h(x_n)\le \limsup_{n\to\infty}\sigma_n=\sigma^\star,
$$
i.e. $x^\star\in \Phi(\sigma^\star)$, and hence $\Phi(\cdot)$ is upper hemicontinuous.\\
\noindent
\underline{Lower hemicontinuity.} Let $\{\sigma_n\}_{n=1}^\infty\subset \mathbb{X}$ be a sequence converging to $\sigma^\star\ge 0$ as $n\to\infty$, and $x^\star\in \Phi(\sigma^\star)$, or equivalently $h(x^\star)\le \sigma^\star$. We claim there exists $\{\sigma_{n_m}\}_{m=1}^\infty\subseteq \{\sigma_n\}_{n=1}^\infty$ and $\{x_m\}_{m=1}^\infty$ such that $x_m\in \Phi(\sigma_{n_m})$ and $x_m\xrightarrow{m\to\infty}x^\star$.

We first consider case $h(x^\star)\le 0$. As $h(x^\star)\le 0$, $x^\star\in \Phi(\sigma)$ for any $\sigma\ge 0$. We choose whatever subsequence $\{\sigma_{n_m}\}_{m=1}^\infty\subseteq \{\sigma_n\}_{n=1}^\infty$ and $x_m=x^\star\in \Phi(\sigma_{n_m}),\,\forall m\in \NN$, the claim is satisfied. As for the case $h(x^\star)>0$, Assumption~\ref{ass:hg_new}-(b)(c) implies that $x^\star$ is not local minimum of $h$. Hence for any $m\in \NN$, $\exists x_m\in \mathcal{X}$ such that $\left|x_m-x^\star\right|\le 1/m$ and $h(x_m)<h(x^\star)$. As $\sigma_n\xrightarrow{n\to\infty} \sigma^\star$, there is a subsequence $\{\sigma_{n_m}\}_{m=1}^\infty$ such that $n_m<n_{m+1}$ and $h(x_m)\le \sigma_m$, or equivalently $x_m\in \Phi(\sigma_{n_m})$, and hence $\Phi(\cdot)$ is lower hemicontinuous.

\end{proof}

\section{PROOF OF THE REMAINING LEMMA AND PROPOSITION}
\subsection{Proof of Lemma~\ref{lem:from apt}}\label{subsec:from apt}
\begin{proof}
As $\mathcal{P}$ is a compact set and $V^\pi_{\cdot}(s)$ is a continuous function for each $s\in\mathcal{S}$, it suffices to show that for all $q\in\mathcal{P}$, $\max_{s,s'\in\mathcal{S}}V^\pi_q(s)-V^\pi_q(s')>0$. Suppose on the contrary, there is $q\in\mathcal{P}$ such that $\max_{s,s'\in\mathcal{S}}V^\pi_q(s)-V^\pi_q(s')=0$, then $\forall s\in\mathcal{S},\,V^\pi_q(s)=\alpha$ for some constant $\alpha\in\mathbb{R}$. As $$
     Q^\pi_{q}(s,a)=r(s,a)+\gamma\sum_{s'}q(s'\vert s,a)V^{\pi}_{q}(s')=r(s,a)+\gamma \alpha,
     $$    
     the definition of $r^\pi(s)$ yields that
     \begin{align*}
          r^\pi(s)&=\sum_{a\in\mathcal{A}}\pi(a\vert s)r(s,a)=\sum_{a\in\mathcal{A}}\pi(a\vert s)(Q^\pi_{q}(s,a)-\gamma \alpha)\\
          &=V^\pi_{q}(s)-\gamma\alpha=(1-\gamma)\alpha.
     \end{align*}
     However, this contradicts Assumption~\ref{apt:general}, where $\max_{s} r^\pi(s)>\min_{s} r^\pi(s)$.
\end{proof}

\subsection{Proof of Proposition~\ref{prop:finite T}}
\label{subsec:finite T}

\begin{proof}
   Suppose not, there exists a fully supported $\bo\in \Sigma$ such that $T_{\bo}(p)=\infty$, or equivalently $T^{-1}_{\bo}(p)=0$, then one has $\sum_{sa}\omega_{sa}\kl_{sa}(p,q^\star)=0$ for some $q^\star\in {\text{cl}}(\alt (p))$, where $\text{cl}(S)$ denotes the closure of a set $S$. As for each $s,a$, $\omega_{sa}>0$ by Assumption~\ref{apt:sampling}, and hence $\kl_{sa}(p,q^\star)=0$. This yields that $p(\cdot\mid s,a)=q^\star(\cdot\mid s,a)$ for all $s,a$. Since the transition probability under $\pi$ and $p$ is the same as the one under $\pi$ and $q^\star$, $V_p^\pi(\brho)=V^\pi_{q^\star}(\brho)$, which however contradicts $q^\star\in {\text{cl}}(\alt (p))$ and the assumption $p\in \Ptest$.
\end{proof}

\section{Experimental Details}\label{app:exp_details}
The simulations presented in this paper were conducted using the following computational environment:
\begin{itemize}
\item Operating system: macOS Sonoma
\item Programming language: Python
\item Processor: Apple M1 Max
\item Memory: 64 GB
\end{itemize}
Uniform sampling was used as the sampling rule. We define the sequence $\zeta_t$ as $\zeta_t = \frac{5}{t^{3/2}}$. In the experiments, the inner optimization is implemented heuristically by taking gradient steps in the reversed MDP and numerically projecting each iterate onto $\bar{\Pi}_\sigma^{\,p}$ using SLSQP. We fixed $L = 400.0$ and capped the maximum value of $M$ at $20$.  

As $|\mathcal{S}|$ and $|\mathcal{A}|$ increased, PTST tended to statistically perform better than the baseline. However, the computational cost grew with problem size; in particular, projection onto $\bar{\Pi}_\sigma^{\,p}$ dominated the runtime. An important future direction is to develop projection-free variants of our method (e.g., conditional gradient updates in the reversed MDP) that avoid Euclidean projections onto a convex set and further reduce computational overhead.

\begin{table}[ht]
    \centering
    \caption{Reward function $r(s,a)$, transition kernel $p(\cdot|s,a)$, and policies $\pi(a|s)$ and $\pi'(a|s)$ for all state-action pairs in the 2-state, 2-action case (${\cal S}={\cal A}=\{0,1\}$).}
    \label{tab:all_mdp_info_2d}
    \begin{tabular}{c|cc}
        \toprule
        \multicolumn{3}{c}{\textbf{Reward function $r(s,a)$}} \\
        \midrule
         $s \backslash a$ & 0 & 1 \\
        \midrule
         0 & $0.50$ & $-0.175$ \\
         1 & $-0.775$ & $1.00$ \\
        \midrule
        \multicolumn{3}{c}{} \\
        \multicolumn{3}{c}{\textbf{Transition kernel $p(\cdot|s,a)$}} \\
        \midrule
         $(s,a)$ & $p(0|s,a)$ & $p(1|s,a)$ \\
        \midrule
        (0, 0) & 0.700 & 0.300 \\
        (0, 1) & 0.400 & 0.600 \\
        (1, 0) & 0.800 & 0.200 \\
        (1, 1) & 0.100 & 0.900 \\
        \midrule
        \multicolumn{3}{c}{} \\
        \multicolumn{3}{c}{\textbf{Policy $\pi(a|s)$}} \\
        \midrule
         $s \backslash a$ & 0 & 1 \\
        \midrule
         0 & 0.150 & 0.850 \\
         1 & 0.507 & 0.493 \\
        \midrule
        \multicolumn{3}{c}{} \\[-1.2ex]
        \multicolumn{3}{c}{\textbf{Another policy $\pi'(a|s)$}} \\
        \midrule
         $s \backslash a$ & 0 & 1 \\
        \midrule
         0 & 0.3848 & 0.6152 \\
         1 & 0.6152 & 0.3848 \\
        \bottomrule
    \end{tabular}
\end{table}

\begin{table}[ht]
    \centering
    \caption{Reward function $r(s,a)$, transition kernel $p(\cdot|s,a)$, and policies $\pi(a|s)$ and $\pi'(a|s)$ for all state-action pairs in the 3-state, 3-action case (${\cal S}={\cal A}=\{0,1,2\}$).}
    \label{tab:all_mdp_info}
    \begin{tabular}{c|ccc}
        \toprule
        \multicolumn{4}{c}{\textbf{Reward function $r(s,a)$}} \\
        \midrule
         $s \backslash a$ & 0 & 1 & 2 \\
        \midrule
         0 & $-0.20$ & $0.02$ & $-0.01$ \\
         1 & $-0.50$ & $-0.01$ & $0.50$ \\
         2 & $-0.01$ & $-0.05$ & $0.20$ \\
        \midrule
        \multicolumn{4}{c}{} \\
        \multicolumn{4}{c}{\textbf{Transition kernel $p(\cdot|s,a)$}} \\
        \midrule
         $(s,a)$ & $p(0|s,a)$ & $p(1|s,a)$ & $p(2|s,a)$ \\
        \midrule
        (0, 0) & 0.3460 & 0.5027 & 0.1513 \\
        (0, 1) & 0.2230 & 0.7014 & 0.0756 \\
        (0, 2) & 0.4077 & 0.3005 & 0.2919 \\
        (1, 0) & 0.2711 & 0.5011 & 0.2277 \\
        (1, 1) & 0.1711 & 0.6011 & 0.2277 \\
        (1, 2) & 0.1711 & 0.1011 & 0.7277 \\
        (2, 0) & 0.2433 & 0.5999 & 0.1568 \\
        (2, 1) & 0.1867 & 0.2998 & 0.5135 \\
        (2, 2) & 0.4033 & 0.0993 & 0.4974 \\
        \midrule
        \multicolumn{4}{c}{} \\
        \multicolumn{4}{c}{\textbf{Policy $\pi(a|s)$}} \\
        \midrule
         $s \backslash a$ & 0 & 1 & 2 \\
        \midrule
         0 & 0.6 & 0.3 & 0.1 \\
         1 & 0.333 & 0.333 & 0.333 \\
         2 & 0.1 & 0.2 & 0.7 \\
        \midrule
        \multicolumn{4}{c}{} \\[-1.2ex]
        \multicolumn{4}{c}{\textbf{Another policy $\pi'(a|s)$}} \\
        \midrule
         $s \backslash a$ & 0 & 1 & 2 \\
        \midrule
         0 & 0.329963 & 0.335487 & 0.334550 \\
         1 & 0.329790 & 0.329798 & 0.340412 \\
         2 & 0.331231 & 0.330005 & 0.338764 \\
        \bottomrule
    \end{tabular}
\end{table}

\begin{table}[ht]
    \centering
    \caption{Reward function $r(s,a)$, transition kernel $p(\cdot|s,a)$, and two policies $\pi(a|s)$ and $\pi'(a|s)$ for all state-action pairs in the 5-state, 5-action case (${\cal S}={\cal A}=\{0,1,2,3,4\}$).}
    \label{tab:all_mdp_info_5}
    \small
    \begin{tabular}{c|ccccc}
        \toprule
        \multicolumn{6}{c}{\textbf{Reward function $r(s,a)$}} \\
        \midrule
         $s \backslash a$ & 0 & 1 & 2 & 3 & 4 \\
        \midrule
         0 & 0.11596 & -0.10323 & 0.07086 & -0.14514 & 0.01885 \\
         1 & -0.08898 & 0.18378 & 0.20909 & 0.18429 & -0.00352 \\
         2 & -0.11392 & 0.23644 & -0.15099 & -0.20320 & -0.23474 \\
         3 & 0.10058 & 0.08980 & 0.00906 & 0.19939 & 0.02957 \\
         4 & 0.11086 & 0.02878 & -0.12984 & 0.17238 & 0.03751 \\
        \midrule
        \multicolumn{6}{c}{} \\
        \multicolumn{6}{c}{\textbf{Transition kernel $p(s'|s,a)$}} \\
        \midrule
         $(s,a)$ & $p(0|s,a)$ & $p(1|s,a)$ & $p(2|s,a)$ & $p(3|s,a)$ & $p(4|s,a)$ \\
        \midrule
        (0, 0) & 0.0191 & 0.2797 & 0.3241 & 0.0813 & 0.2958 \\
        (0, 1) & 0.2279 & 0.2631 & 0.0458 & 0.2566 & 0.2066 \\
        (0, 2) & 0.1418 & 0.2505 & 0.2561 & 0.2799 & 0.0718 \\
        (0, 3) & 0.3117 & 0.1916 & 0.0851 & 0.1691 & 0.2424 \\
        (0, 4) & 0.1199 & 0.6589 & 0.2133 & 0.0040 & 0.0038 \\
        (1, 0) & 0.1452 & 0.3076 & 0.0715 & 0.1816 & 0.2941 \\
        (1, 1) & 0.4654 & 0.0252 & 0.2148 & 0.2654 & 0.0292 \\
        (1, 2) & 0.2123 & 0.0780 & 0.2095 & 0.2257 & 0.2745 \\
        (1, 3) & 0.2350 & 0.1905 & 0.1488 & 0.1254 & 0.3003 \\
        (1, 4) & 0.0091 & 0.3348 & 0.0134 & 0.1328 & 0.5099 \\
        (2, 0) & 0.2699 & 0.3663 & 0.2291 & 0.0208 & 0.1139 \\
        (2, 1) & 0.2535 & 0.2019 & 0.1512 & 0.2041 & 0.1893 \\
        (2, 2) & 0.3340 & 0.2574 & 0.1303 & 0.1418 & 0.1365 \\
        (2, 3) & 0.1428 & 0.1237 & 0.1114 & 0.0747 & 0.5474 \\
        (2, 4) & 0.1530 & 0.3078 & 0.1651 & 0.3379 & 0.0362 \\
        (3, 0) & 0.0043 & 0.3403 & 0.1235 & 0.0826 & 0.4493 \\
        (3, 1) & 0.0870 & 0.3120 & 0.0742 & 0.2682 & 0.2587 \\
        (3, 2) & 0.1755 & 0.2717 & 0.1635 & 0.1257 & 0.2637 \\
        (3, 3) & 0.2272 & 0.1819 & 0.2460 & 0.0933 & 0.2516 \\
        (3, 4) & 0.2717 & 0.1775 & 0.0811 & 0.1830 & 0.2868 \\
        (4, 0) & 0.2812 & 0.0261 & 0.0534 & 0.4150 & 0.2243 \\
        (4, 1) & 0.2381 & 0.2541 & 0.1767 & 0.2693 & 0.0617 \\
        (4, 2) & 0.4520 & 0.1074 & 0.0020 & 0.1489 & 0.2897 \\
        (4, 3) & 0.3384 & 0.0184 & 0.1746 & 0.3144 & 0.1541 \\
        (4, 4) & 0.0686 & 0.1741 & 0.2139 & 0.1872 & 0.3563 \\
        \midrule
        \multicolumn{6}{c}{} \\
        \multicolumn{6}{c}{\textbf{Policy $\pi(a|s)$}} \\
        \midrule
         $s \backslash a$ & 0 & 1 & 2 & 3 & 4 \\
        \midrule
         0 & 0.1535 & 0.2298 & 0.0998 & 0.2521 & 0.2648 \\
         1 & 0.2159 & 0.2917 & 0.1054 & 0.0903 & 0.2967 \\
         2 & 0.0452 & 0.0699 & 0.1839 & 0.3681 & 0.3329 \\
         3 & 0.2078 & 0.3493 & 0.0826 & 0.2214 & 0.1389 \\
         4 & 0.2311 & 0.1292 & 0.2522 & 0.2173 & 0.1701 \\
        \midrule
        \multicolumn{6}{c}{} \\[-1.2ex]
        \multicolumn{6}{c}{\textbf{Another policy $\pi'(a|s)$}} \\
        \midrule
         $s \backslash a$ & 0 & 1 & 2 & 3 & 4 \\
        \midrule
         0 & 0.1387 & 0.2651 & 0.1637 & 0.3034 & 0.1291 \\
         1 & 0.2705 & 0.1384 & 0.1378 & 0.1367 & 0.3167 \\
         2 & 0.1471 & 0.1155 & 0.1624 & 0.1891 & 0.3859 \\
         3 & 0.1489 & 0.1512 & 0.2145 & 0.1346 & 0.3508 \\
         4 & 0.1398 & 0.2038 & 0.3177 & 0.1403 & 0.1984 \\
        \bottomrule
    \end{tabular}
\end{table}

\end{document}